\documentclass[journal,twoside,web]{ieeecolor}
\usepackage{jsen}
\usepackage{cite}
\usepackage{amsmath,amssymb,amsfonts}
\let\oldsum\sum
\let\oldprod\prod
\renewcommand{\sum}{\textstyle\oldsum}
\renewcommand{\prod}{\textstyle\oldprod}
\usepackage{algorithmic}
\usepackage{graphicx}
\usepackage{textcomp}
\usepackage{wrapfig}
\usepackage{amsmath,amsfonts}
\usepackage{algorithmic}
\usepackage[ruled,linesnumbered]{algorithm2e}
\SetKwInOut{KwParam}{Parameter}
\usepackage{array}
\usepackage[caption=false,font=normalsize,labelfont=sf,textfont=sf]{subfig}
\usepackage{textcomp}
\usepackage[colorlinks=true, linkcolor=black, citecolor=black, urlcolor=black]{hyperref}
\usepackage{orcidlink}
\usepackage{stfloats}
\usepackage{url}
\usepackage{graphicx}
\usepackage{cite}
\usepackage{booktabs}
\usepackage{multirow}
\usepackage{multicol}
\usepackage{makecell}
\usepackage{tabularx}
\usepackage[markup=underlined]{changes} 
\usepackage{float}
\usepackage{tikz}
\usepackage{pgfplots}
\usepackage{amssymb}
\usepackage{xtab}
\usepackage{soul,color,xcolor}
\newtheorem{remark}{Remark}
\newtheorem{theorem}{Theorem}

\renewenvironment{proof}{{\itshape \quad Proof:\;\,}}{\hfill $\square$\par}
\pgfplotsset{compat=1.18}
\usepackage{etoolbox}
\makeatletter
\patchcmd{\@maketitle}
  {\addvspace{0.5\baselineskip}\egroup}
  {\addvspace{-1.8\baselineskip}\egroup}
  {}
  {}

\def\BibTeX{{\rm B\kern-.05em{\sc i\kern-.025em b}\kern-.08em
    T\kern-.1667em\lower.7ex\hbox{E}\kern-.125emX}}
\definecolor{abstractbg}{rgb}{0.89804,0.94510,0.83137}
\begin{document}

\thispagestyle{empty} 
\begin{center}
\vspace*{2cm} 
\textbf{\Large IEEE Copyright Notice}
\vspace{1cm}

\begin{minipage}{0.9\textwidth}
\footnotesize
© 2025 IEEE. Personal use of this material is permitted. Permission from IEEE must be obtained for all other uses, in any current or future media, including reprinting/republishing this material for advertising or promotional purposes, creating new collective works, for resale or redistribution to servers or lists, or reuse of any copyrighted component of this work in other works.
\end{minipage}

\vspace{1cm}
DOI: \texttt{10.1109/jsen.2025.3582282} 
\end{center}

\title{An Adaptive Sliding Window Estimator for Positioning of Unmanned Aerial Vehicle Using a Single Anchor}
\author{Kaiwen Xiong\orcidlink{0009-0000-4575-0663}, Sijia Chen\orcidlink{0000-0002-4081-0134}, Wei Dong\orcidlink{0000-0003-2640-1585}
\thanks{This work was supported in part by the Shanghai Rising-Star Program under Grant 22QA1404400  and in part by the National Natural Science Foundation of China Grant 51975348.}
\thanks{Kaiwen Xiong, Sijia Chen and Wei Dong are with the State Key Laboratory of Mechanical System and Vibration, School of Mechanical Engineering, Shanghai Jiaotong University, Shanghai, 200240, China. Corresponding author: Wei Dong, E-mail: dr.dongwei@sjtu.edu.cn.}}

\IEEEtitleabstractindextext{
\fcolorbox{abstractbg}{abstractbg}{
\begin{minipage}{\textwidth}
\begin{wrapfigure}[22]{r}{3in}
\includegraphics[width=3in]{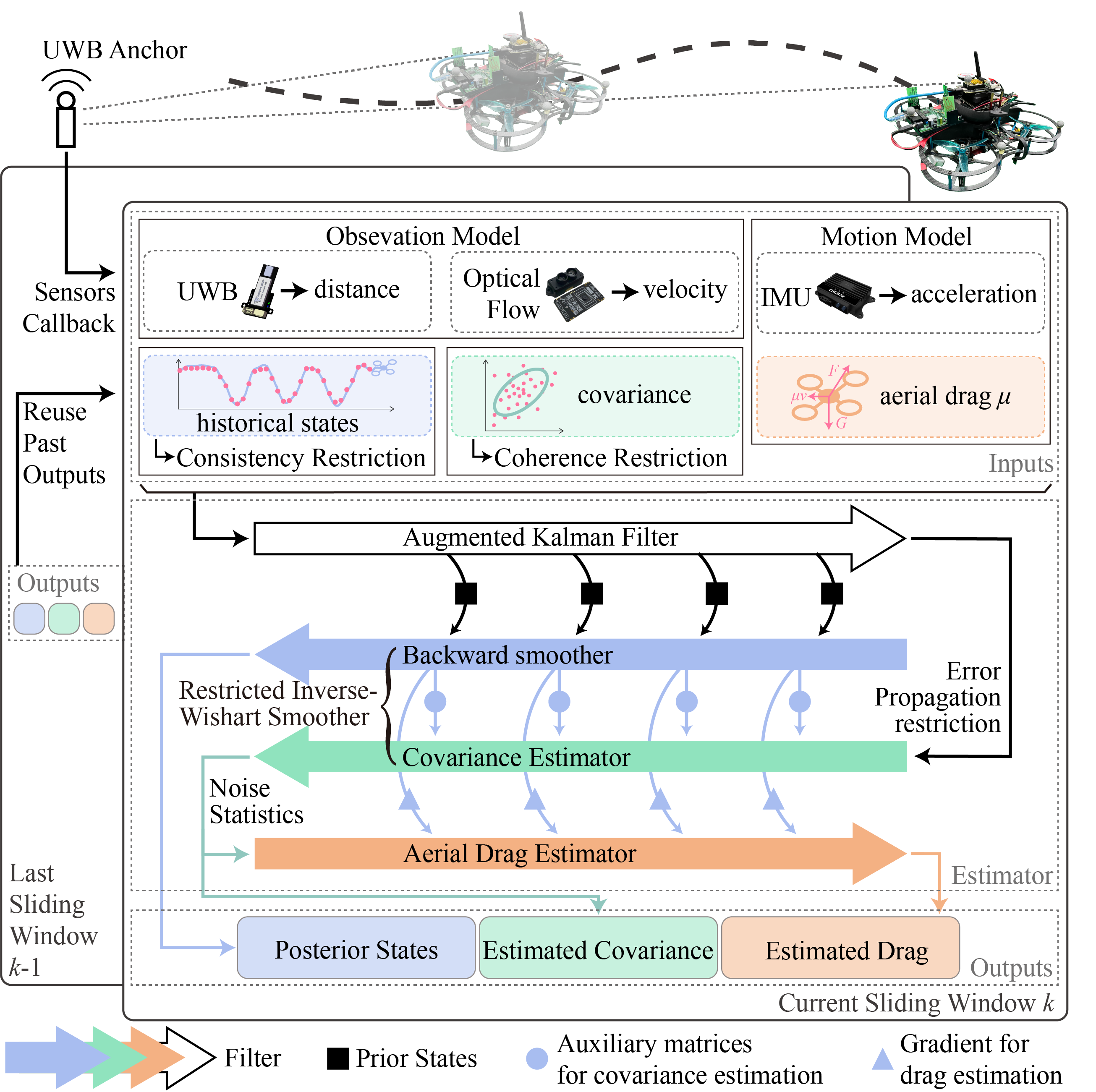}
\end{wrapfigure}
\begin{abstract}
Localization using a single range anchor combined with onboard optical-inertial odometry offers a lightweight solution that provides multidimensional measurements for the positioning of unmanned aerial vehicles.
Unfortunately, the performance of such lightweight sensors varies with the dynamic environment, and the fidelity of the dynamic model is also severely affected by environmental aerial flow.
To address this challenge, we propose an adaptive sliding window estimator equipped with an estimation reliability evaluator, where the states, noise covariance matrices and aerial drag are estimated simultaneously.
The aerial drag effects are first evaluated based on posterior states and covariance.
Then, an augmented Kalman filter is designed to pre-process multidimensional measurements and inherit historical information.
Subsequently, an inverse-Wishart smoother is employed to estimate posterior states and covariance matrices.
To further suppress potential divergence, a reliability evaluator is devised to infer estimation errors.
We further determine the fidelity of each sensor based on the error propagation.
Extensive experiments are conducted in both standard and harsh environments, demonstrating the adaptability and robustness of the proposed method.
The root mean square error reaches 0.15 m, outperforming the state-of-the-art approach.
Real-world close-loop control experiments are additionally performed to verify the estimator's competence in practical application.
\end{abstract}

\end{minipage}}}

\maketitle

\section{Introduction}
\label{secIntro}
\IEEEPARstart{u}{nmanned} aerial vehicles (UAVs) are increasingly employed to replace manpower in harsh environments \cite{8877114}.
To achieve autonomous execution and efficient operation, high-precision positioning and navigation of UAVs are urgently demanded \cite{9919741}.
Unfortunately, in challenging scenarios such as cave exploration \cite{10318478}, facilities maintenance \cite{9615423}, environmental interference and rapid changes underscore the need for further refinement of positioning methods.\par
Typically, to achieve high-precision positioning, many works leverage peripheral equipment, including GNSS \cite{10026812,9739235}, motion capture system \cite{7840347}, pre-established beacon system \cite{8869587} and assistant robot \cite{9650840}.
However, GNSS navigation often becomes infeasible due to signal obstruction, as highlighted by several studies \cite{10190013}, while deploying additional peripheral equipment may be limited by resource constraints in numerous applications.
Hence, onboard sensors have garnered significant interest.
Two primary categories include visual odometry (VO), which encompasses cameras \cite{10295773, 10132859} and lidar \cite{8949363}, albeit requiring substantial computational resources.
The alternative category comprises non-visual methods such as ultra-wideband radio (UWB) \cite{9627174}, inertial measurement units (IMU), or a combination of both \cite{10083062}, which are lightweight but usually have limited information and inadequate positioning precision in complex situations.
Under this premise, integrated systems like UWB-VO \cite{10318307}, visual-inertial odometry \cite{10242349}, and vision-IMU-UWB odometry \cite{10131904,9350155} effectively mitigate drawbacks from both categories.
To handle environmental interference and achieve better positioning precision, it is critical to distinguish varying degrees of sensor failure and properly exploit return data according to sensor fidelity. \par
Adaptive algorithms are developed to tackle this problem. 
They concurrently estimate states and noise's statistical properties based on optimal estimators. 
Among various methods, one category is empirical adjustment incorporating deliberately designed factors \cite{Jiang2021} and fuzzy system \cite{8997181}. The other category explores statistical approaches such as exploiting moment estimation of sample covariance \cite{9077705,10154170} and special prior distributions \cite{9748015,9944196}. 
The methods above demonstrate their competence in slowly changing environments. 
However, in harsh scenarios, current estimators struggle to assess estimation errors synchronously, rendering them susceptible to faulty data due to interference. 
Furthermore, these estimators can hardly adapt to sudden environmental changes, which increases the risk of potential divergence, underscoring the urgent need for more robust solutions.\par
Moreover, the variation in aerial drag presents another facet of environmental change. 
Accounting for aerial drag is crucial in the motion modeling of UAVs \cite{9033633}, significantly enhancing the precision of position estimation \cite{10184165}. 
Many studies have studied the estimation of aerial drag coefficients \cite{9292966, 10341818, 9013062}, but due to the complexity of aerodynamics and limited computational resources, more efficient methods are deserved in practical application.\par
To address the above problems, we elaborately develop a novel method called \textit{restricted adaptive sliding window estimator (RASWE)}, which simultaneously operates state estimation, covariance matrix adaptation, and aerial drag adjustment.
Considering sensor configuration, an optical flow (OF) sensor and IMU are employed to compensate for the limited operation range and single data dimension of UWB.
Then, the dynamics are reformulated considering aerial drag effects.
An augmented Kalman filter pre-processes all data from the sensor system and reuses historical output states under coherence restriction.
The processed data are sent to the backward smoother, yielding posterior states utilized to estimate noise covariance matrices under the control of error propagation restriction.
A cost function is formulated based on the idea that the dynamic model with proper aerial drag matrix derives the prior state that approaches the posterior state, which helps to adjust the aerial drag matrix via gradient descent.
All outputs within one sliding window are reused as inputs at the next timestep.
Finally, experiments are conducted in both common and harsh environments to demonstrate the adaptability of the proposed method. 
Additionally, we design various ablation experiments to validate the effectiveness of each design and perform real-world close-loop control experiments to further verify the ability of the proposed estimator.\par
The major contributions of our work are threefold:
1)
We propose an augmented Kalman filter to estimate the position of UAV using a single anchor efficiently.
2)
We elaborately design an error propagation matrix as an online error inspector to assess estimation performance and develop a restricted inverse-Wishart smoother to derive posterior states and adjust noise covariance synchronously.
3)
We formulate a cost function based on the idea that the prior state obtained by dynamics should approach the posterior state, and we construct a novel aerial drag estimator via the gradient descent method to enhance estimation performance.
\newpage
\section{Preliminaries}\label{secPre}
\subsection{Notations}
In this article, an $m$ by $n$ matrix, $\boldsymbol{A}\in\mathbb{R}^{m\times n}$, 
is referred to by capital bold letter, whereas vector of dimension $n$, $\boldsymbol{x}\in\mathbb{R}^n$, is denoted by a lowercase bold letter. 
Identity and zero matrix of dimension $m\times n$ are represented as $\boldsymbol{I}_{m\times n}$ and $\boldsymbol{0}_{m\times n}$ respectively, and to simplify, the square ones of size $n$ are abbreviated as $\boldsymbol{I}_{n}$ and $\boldsymbol{0}_{n}$ respectively. 
If their dimension is not emphasized, the index will be omitted.
A quantity with index, $(\cdot)_k$, indicates it is at timestep $k$, and the transposition and inverse of matrix $(\cdot)$ are $(\cdot)^\mathrm{T}$ and $(\cdot)^{-1}$ respectively.
Also, prior and posterior quantities usually are decorated with superscripts, $\check{(\cdot)}$ and $\hat{(\cdot)}$ respectively.
\subsection{Basic Concepts}
In this work, we consider a discrete-time, linear and time-varying model as follows:
    \begin{align}
        \text{motion model: }&\boldsymbol{x}_k=\boldsymbol{A}_{k-1}\boldsymbol{x}_{k-1}+\boldsymbol{u}_k+\boldsymbol{w}_k \label{22basicMODEL1}\\
        \text{observation model: }&\boldsymbol{y}_k=\boldsymbol{C}_{k}\boldsymbol{x}_{k}+\boldsymbol{n}_k \label{22basicMODEL2}
    \end{align} 
where $k$ is the discrete-time index with its maximum $K$. 
The state vector $\boldsymbol{x}_k=[\boldsymbol{p}_k^\mathrm{T},\boldsymbol{v}_k^\mathrm{T}]^\mathrm{T}\in\mathbb{R}^6$ consists of 3-dimensional position and velocity, which are the focus of our estimation.
The measurements from the sensor system are $\boldsymbol{y}_k\in\mathbb{R}^{m}$, where $m$ is the size of it.
The transition matrix $\boldsymbol{A}_{k-1}\in\mathbb{R}^{6\times6}$ and the observation matrix $\boldsymbol{C}_{k}\in\mathbb{R}^{m\times6}$ are already known based on historical knowledge.
The net input is $\boldsymbol{u}_k$, and it is also written as $\boldsymbol{u}_k=\boldsymbol{B}_k\boldsymbol{i}_k$, where $\boldsymbol{B}_k\in\mathbb{R}^{6\times s_i}$ is a known control matrix and $\boldsymbol{i}_k\in\mathbb{R}^{s_i}$ is the original input of size $s_i$.\par
The processing and observation noises are formulated as unbiased Gaussian distribution, i.e. $\boldsymbol{w}_k\sim \mathbf{N}(\boldsymbol{0},\boldsymbol{Q}_k)$ and $\boldsymbol{n}_k\sim \mathbf{N}(\boldsymbol{0},\boldsymbol{R}_k)$.
Also, they are assumed to be uncorrelated.
This means $\mathbb{E}[\boldsymbol{w}_i\boldsymbol{w}_j^\mathrm{T}]=\delta_{ij}\boldsymbol{Q}_i$, $\mathbb{E}[\boldsymbol{n}_i\boldsymbol{n}_j^\mathrm{T}]=\delta_{ij}\boldsymbol{R}_i$ and $[\boldsymbol{w}_i\boldsymbol{n}_j^\mathrm{T}]=\boldsymbol{0}$,
where $\mathbb{E}[\cdot]$ denotes the mathematical expectation and $\delta_{ij}$ is Kronecker-$\delta$ function.
And $\mathbf{N}(\boldsymbol{\nu},\boldsymbol{C})$ represents multivariate Gaussian distribution with mean $\boldsymbol{\nu}$ and covariance $\boldsymbol{C}$.
Also in our work, the state $\boldsymbol{x}_k$ is assumed to be Gaussian, i.e. $\boldsymbol{x}_k\sim\mathbf{N}(\check{\boldsymbol{x}},\check{\boldsymbol{P}})$ or $\boldsymbol{x}_k\sim\mathbf{N}(\hat{\boldsymbol{x}},\hat{\boldsymbol{P}})$.Moreover, we say a random $n \times n$ matrix follows inverse-Wishart distribution, $\mathbf{IW}(\sigma,\boldsymbol{\Sigma})$, with degree of freedom (DoF) $\sigma$ and scale matrix parameter $\boldsymbol{\Sigma}$.
\section{Methodology}\label{secMethod}
\subsection{Dynamics and Observation Model}
Various interference from harsh environments poses a daunting challenge to sensor configuration.
The practical application requires both light-weighting and resilience to interference.
Under this premise, we deliberately design a sensor system consisting of a single UWB anchor, IMU and OF sensor.
The different working conditions for each chosen sensor ensure that a single type of interference, such as dim light or long distance, only affects a specific sensor, leaving the others safe to function effectively.\par
The IMU collects normalized linear acceleration $\boldsymbol{a}$ and normalized quaternion $\boldsymbol{q}$.
The actual acceleration is obtained by $\boldsymbol{i}=g\boldsymbol{a}+[0,0,-g]^\mathrm{T}$,
where $g\simeq 9.8$ is gravitational acceleration and $\boldsymbol{R}$ is the spatial rotation matrix corresponding to $\boldsymbol{q}$ \cite{salamin1979application}.
By pre-integration, the dynamics are formulated as:
\begin{equation}
    \label{31motionModel}
    \begin{aligned}
        \boldsymbol{x}_{k}&=\begin{bmatrix}
            \boldsymbol{I}_3 & \mathrm{dt}\cdot\boldsymbol{I}_3\\
            \boldsymbol{0}_3 & \boldsymbol{I}_3-\mathrm{dt}\cdot\boldsymbol{\mu}\\
        \end{bmatrix}\boldsymbol{x}_{k-1}+\begin{bmatrix}
            \frac{1}{2}\mathrm{dt}^2\\
            \mathrm{dt}
        \end{bmatrix}\otimes\boldsymbol{i}_k+\mathbb{E}[\boldsymbol{w}_k]\\
        &=\boldsymbol{A}_{k-1}\boldsymbol{x}_{k-1}+\boldsymbol{u}_k
    \end{aligned}
\end{equation}
where $\otimes$ is Kronecker product and $\mathrm{dt}$ is the time between timestep $k-1$ and $k$.
The $\boldsymbol{\mu}\in\mathbb{R}^{3\times3}$ is aerial drag matrix.
It will be discussed thoroughly in Subsection~\ref{secMUE}, but before that, it is assumed to be already known. \par
Typically, UWB is a non-linear sensor. 
For the computational efficiency, it is linearized by:
\begin{equation}
    \label{31UWBlinear}
    {}^\text{UWB}y_k=\begin{bmatrix}
        \frac{\tilde{\boldsymbol{p}}_k^\mathrm{T}}{||\tilde{\boldsymbol{p}}_k||_2}&\boldsymbol{0}_{1\times3}
    \end{bmatrix}\boldsymbol{x}_{k}
\end{equation}
where $||\cdot||_2$ is the 2-norm (a.k.a. Euclidean norm).
The $\tilde{\boldsymbol{p}}_k$ is an approximate position calculated via \eqref{31motionModel}:
\begin{equation}
    \label{31approP}
    \tilde{\boldsymbol{p}}_k=\begin{bmatrix}
        \boldsymbol{I}_3&\boldsymbol{0}_3
    \end{bmatrix}(\boldsymbol{A}_{k-1}\check{\boldsymbol{x}}_{k-1}+\boldsymbol{u}_k)
\end{equation}
where $\check{\boldsymbol{x}}_{k-1}$ is the output of the last sliding window, and thereby it is known here.
The UWB applies a global constraint on our estimator, which suppresses the error accumulation and overcomes the most common interference, vision degradation.\par
\begin{remark}
    For convenience, we set the origin at the same place as the UWB anchor.
    Nevertheless, in practical application, since the UWB anchor is fixed w.r.t. world frame, its position can be pre-calibrated easily.
    If we choose a different origin, by subtracting UWB's position $\boldsymbol{p}_\text{UWB}$ from both sides of \eqref{31approP} and adding $\tilde{\boldsymbol{p}}_k^\mathrm{T}\boldsymbol{p}_\text{UWB}/||\tilde{\boldsymbol{p}}_k||_2$ to ${}^\text{UWB}y_k$, it achieves the same effect as \eqref{31UWBlinear}. 
\end{remark}\par
The velocity ${}^\text{OF}\boldsymbol{y}_k = [{}^\text{H}\boldsymbol{y}_k^\mathrm{T},{}^\text{V}\boldsymbol{y}_k^\mathrm{T}]^\mathrm{T}$ measured by OF unit has no range limitation, which compensates for the disadvantage of UWB.
Typically, the OF method only obtains 2-dimensional velocity, which is horizontal velocity ${}^\text{H}\boldsymbol{y}_k$.
The vertical velocity ${}^\text{V}\boldsymbol{y}_k$ is obtained by calculating the height difference from a laser unit between each timestep.
The observation model is formulated as follows:
\begin{equation}
    \label{31observationModel}
    \boldsymbol{y}_k=\begin{bmatrix}
        {}^\text{UWB}y_k\\{}^\text{OF}\boldsymbol{y}_k
    \end{bmatrix}=\begin{bmatrix}
            \tilde{\boldsymbol{p}}_k^\mathrm{T}/||\tilde{\boldsymbol{p}}_k||_2 & \boldsymbol{0}_{1\times 3}\\
            \boldsymbol{0}_3 & \boldsymbol{I}_3
        \end{bmatrix}\boldsymbol{x}_{k}=\boldsymbol{C}_k\boldsymbol{x}_{k}.
\end{equation}\par
This system's advantage is that the observation model's position and velocity parts compensate for each other's weaknesses, which guarantees the robustness of measurements and enhances adaptability in changeable environments.
\subsection{Augmented Kalman Filter}\label{sec:AKF}
This part is the basis of the proposed estimator, which pre-processes all measurements collected by sensors and information calculated previously by the estimator itself.
We exert Kalman filter (KF) during a sliding window.
Since every sliding window overlaps each other, we redundantly calculate the posterior state at each timestep many times, as is displayed in Fig.~\ref{fig:SW}.
To fully exploit this redundancy, we use an interesting trick to consider our KF as some other kind of sensor.
Thus, the measurements, as well as corresponding observation and covariance matrices, are augmented:
\begin{subequations}
\label{32yCR}
\begin{align}
    \label{32yCR1}
    \tilde{\boldsymbol{y}}_j&=[\boldsymbol{y}_j^\mathrm{T}, \check{\boldsymbol{x}}_j^\mathrm{T}]^\mathrm{T},\;j=1,2,\cdots,k_w-1\\
    \label{32yCR2}
    \tilde{\boldsymbol{C}}_j&=[\boldsymbol{C}_j^\mathrm{T}, \boldsymbol{I}_6]^\mathrm{T},\;j=1,2,\cdots,k_w-1\\
    \label{32yCR3}
    \tilde{\boldsymbol{R}}_j&=\mathrm{diag}(\boldsymbol{R}_j,\check{\boldsymbol{P}}_j),\;j=1,2,\cdots,k_w-1
\end{align}
\end{subequations}
where $\mathrm{diag}(\boldsymbol{A},\boldsymbol{B})$ is blockwise diagonal matrix with main diagonal blocks $\boldsymbol{A}$ and $\boldsymbol{B}$.
If $j=k_w$, the $\tilde{\boldsymbol{y}}_j$, $\tilde{\boldsymbol{C}}_j$ and $\tilde{\boldsymbol{R}}_j$ remain the same as their original form.\par
\begin{figure}[tpb]
    \centering
    \includegraphics[width=0.95\linewidth]{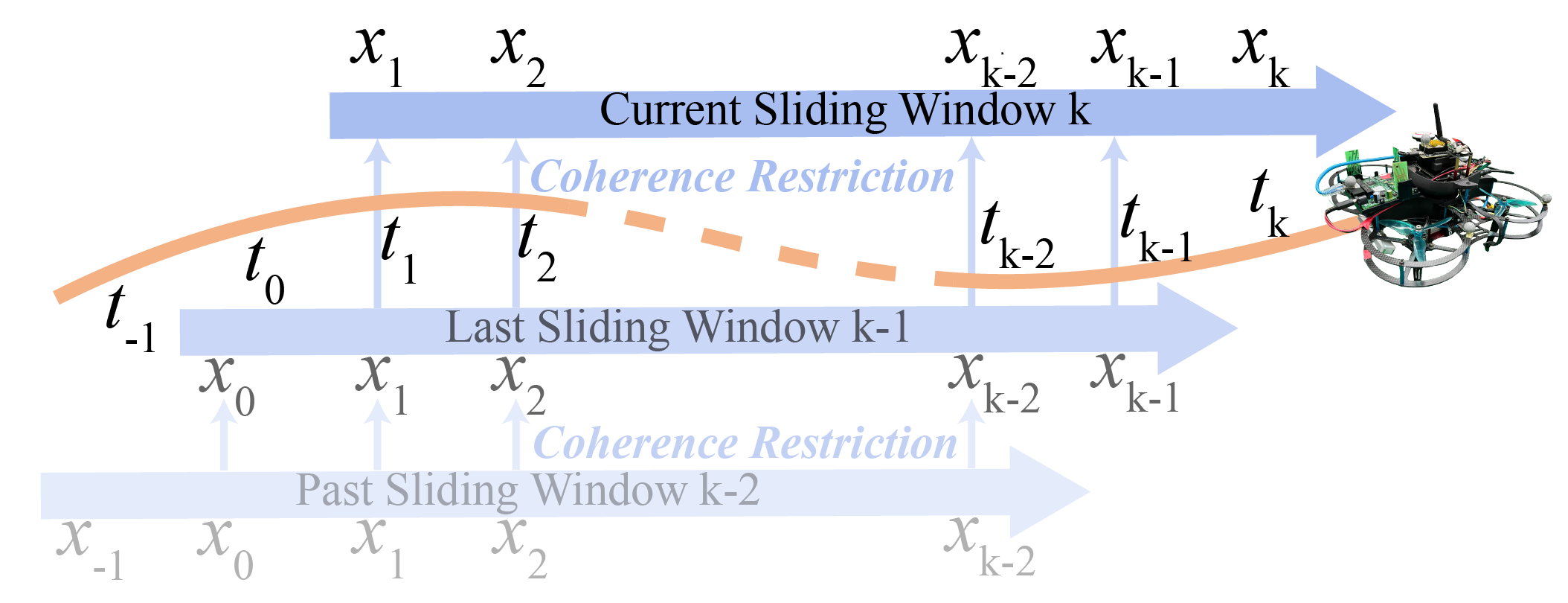}
    \vspace{-7pt}
    \caption{Coherence restriction to exploit redundancy emerging from overlapping sliding windows.}
    \label{fig:SW}
    \vspace{-2pt}
\end{figure}
In this way, we construct our augmented KF as follows:
\begin{subequations}
    \label{32KFforward}
    \begin{align}
        \check{\boldsymbol{P}}_{f,j}&=\boldsymbol{A}_{j-1}\hat{\boldsymbol{P}}_{f,j-1}\boldsymbol{A}_{j-1}^\mathrm{T}+\boldsymbol{Q}_j\label{32KFforward1}\\
        \check{\boldsymbol{x}}_{f,j}&=\boldsymbol{A}_{j-1}\hat{\boldsymbol{x}}_{f,j-1}+\boldsymbol{u}_{j}\label{32KFforward2}\\
        \boldsymbol{K}_j&=\check{\boldsymbol{P}}_{f,j}\tilde{\boldsymbol{C}}_j^\mathrm{T}(\tilde{\boldsymbol{C}}_j\check{\boldsymbol{P}}_{f,j}\tilde{\boldsymbol{C}}_j^\mathrm{T}+\tilde{\boldsymbol{R}}_j)^{-1}\label{32KFforward3}\\
        \hat{\boldsymbol{P}}_{f,j}&=(\boldsymbol{I}-\boldsymbol{K}_j\tilde{\boldsymbol{C}}_j)\check{\boldsymbol{P}}_{f,j}\label{32KFforward4}\\
        \hat{\boldsymbol{x}}_{f,j}&=\check{\boldsymbol{x}}_{f,j}+\boldsymbol{K}_j(\tilde{\boldsymbol{y}}_j-\tilde{\boldsymbol{C}}_j\check{\boldsymbol{x}}_{f,j})\label{32KFforward5}
    \end{align}
\end{subequations}
where $j=1,2,\cdots,k_w$.
The filter is initialized by $\hat{\boldsymbol{P}}_{f,0}=\check{\boldsymbol{P}}_0$ and $\hat{\boldsymbol{x}}_{f,0}=\check{\boldsymbol{x}}_0$.
The label $f$ means the $\boldsymbol{x}$ and $\boldsymbol{P}$ are only the intermediate temporary variables used to support the forward process, not the estimated posterior state.
We will derive posterior states at the beginning of Subsection~\ref{secIWS}.\par
Mathematically \cite{roboEst}, the above KF can be reformulated into the maximum a posteriori (MAP) solving process, where the covariance matrices are treated as weights to balance the previous estimation, the motion modal and the observation model.
The following section will introduce how to adjust these weights online.
The necessity of this overlapping sliding window form is further discussed in Subsection~\ref{secAna}.\par
The augmentation in \eqref{32yCR} can also be explained as coherence restriction, forcing the estimated state at the same timestep of nearby sliding windows to be as close as possible, which enhances the numeric stability of the calculation.
Additionally, all covariance matrices $\boldsymbol{Q}$'s and $\boldsymbol{R}$'s are set to be the same in each sliding window, which is called consistency restriction.\par
Compared to the standard Kalman filter, the proposed method has two main distinctions.
First, it extends the conventional step-by-step filtering framework into a sliding window-based MAP estimation, where a Kalman-style update is employed within each window to balance estimation accuracy and computational efficiency.
Second, it augments the estimation formulation by explicitly incorporating past state estimates within the window, thereby reducing numerical inconsistency and enhancing temporal smoothness.
This design reflects a key structural distinction between our method and conventional Kalman filter variations.\par
\subsection{Restricted Inverse-Wishart Smoother}\label{secIWS}
In this subsection, we combine the backward smoother and covariance estimator into a single compact algorithm, restricted inverse-Wishart smoother.
We first conduct a backward smoother as follows:
\begin{subequations}
    \label{33backward}
    \begin{align}
        \boldsymbol{G}_j&=\hat{\boldsymbol{P}}_{f,j-1}\boldsymbol{A}_{j-1}^\mathrm{T}\check{\boldsymbol{P}}_{f,j}^{-1}\label{33backward1}\\
        \hat{\boldsymbol{x}}_{j-1}&=\hat{\boldsymbol{x}}_{f,j-1}+\boldsymbol{G}_j(\hat{\boldsymbol{x}}_j-\check{\boldsymbol{x}}_{f,j})\label{33backward2}\\
        \hat{\boldsymbol{P}}_{j-1}&=\hat{\boldsymbol{P}}_{f,j-1}+\boldsymbol{G}_j(\hat{\boldsymbol{P}}_j-\check{\boldsymbol{P}}_{f,j})\boldsymbol{G}_j^\mathrm{T}\label{33backward3}
    \end{align}
\end{subequations}
where $j=k_w,\cdots,2,1$, with the initial step: $\hat{\boldsymbol{P}}_{k_w}=\hat{\boldsymbol{P}}_{f,k_w}$, $\hat{\boldsymbol{x}}_{k_w}=\hat{\boldsymbol{x}}_{f,k_w}$.
Through this, we obtain our deserved posterior states, which are both the output of the estimator and the basis for the estimation of covariance matrices and aerial drag effect.\par
As for noise covariance matrices, we first assume they obey inverse-Wishart distribution:
\begin{subequations}
    \label{33QRlast}
    \begin{align}
        p(\hat{\boldsymbol{Q}}_{k-1}|\hat{\boldsymbol{x}}_{0:k_w-1})&=\mathbf{IW}(\hat{\boldsymbol{Q}}_{k-1};\hat{\phi}_{k-1},\hat{\boldsymbol{\Phi}}_{k-1})\label{33QRlast1}\\
        p(\hat{\boldsymbol{R}}_{k-1}|\hat{\boldsymbol{x}}_{0:k_w-1})&=\mathbf{IW}(\hat{\boldsymbol{R}}_{k-1};\hat{\psi}_{k-1},\hat{\boldsymbol{\Psi}}_{k-1})\label{33QRlast2}
    \end{align}
\end{subequations}
which are known currently since its label $k-1$ indicates it is the posterior estimation obtained by the last sliding window. 
The parameters are chosen empirically for the initial step, and hence, the corresponding probability density functions (PDFs) are also known.
In the following equations of this subsection, we use $n=6$ and $m=4$ to represent the dimension of state and measurements, respectively, for a consistent form as inverse-Wishart distribution. \par
According to knowledge of probability theory \cite{gelman2013bayesian}, if the columns of the sample $\boldsymbol{X}=\begin{bmatrix}
    \boldsymbol{x}_1,\boldsymbol{x}_2,\cdots,\boldsymbol{x}_K
\end{bmatrix}$ are independent and identically distributed $n$-dimensional Gaussian variables conforming to $\mathbf{N}(\boldsymbol{0},\boldsymbol{C})$, with $C\sim\mathbf{IW}(\sigma,\boldsymbol{\Sigma})$, then the conditional PDF is $p(\boldsymbol{C}|\boldsymbol{X})=\mathbf{IW}(\boldsymbol{C};\sigma+K,\boldsymbol{\Sigma}+\boldsymbol{X}\boldsymbol{X}^\mathrm{T})$.
This is the conjugation nature of inverse-Wishart distribution, and to exploit this, we formulate two auxiliary matrices:
\begin{subequations}
    \label{33QRauxiliary}
    \begin{align}
        \tilde{\boldsymbol{\Phi}}_j=&\;\mathbb{E}[\boldsymbol{e}_{1,j}\boldsymbol{e}_{1,j}^\mathrm{T}]
        =\hat{\boldsymbol{P}}_j-\boldsymbol{A}_{j-1}\boldsymbol{G}_j\hat{\boldsymbol{P}}_j
        -\boldsymbol{G}_j\hat{\boldsymbol{P}}_j\boldsymbol{A}_{j-1}^\mathrm{T}\notag\\
        &\;+\boldsymbol{A}_{j-1}\hat{\boldsymbol{P}}_{j-1}\boldsymbol{A}_{j-1}^\mathrm{T}+\boldsymbol{e}_{1,j}\boldsymbol{e}_{1,j}^\mathrm{T}\label{33QRauxiliary1}\\
        \tilde{\boldsymbol{\Psi}}_j=&\;\mathbb{E}[\boldsymbol{e}_{2,j}\boldsymbol{e}_{2,j}^\mathrm{T}]
        =\boldsymbol{C}_j\hat{\boldsymbol{P}}_j\boldsymbol{C}_j^\mathrm{T}+\boldsymbol{e}_{2,j}\boldsymbol{e}_{2,j}^\mathrm{T}\label{33QRauxiliary2}
    \end{align}
\end{subequations}
where the error terms are $\boldsymbol{e}_{1,j}=\hat{\boldsymbol{x}}_j-\boldsymbol{A}_{j-1}\hat{\boldsymbol{x}}_{j-1}-\boldsymbol{u}_j$ and $\boldsymbol{e}_{2,j}=\boldsymbol{y}_j-\boldsymbol{C}_j\hat{\boldsymbol{x}}_j$.
By applying the conjugation nature, we obtain the posterior noise covariance matrices as follows:
\begin{subequations}
    \label{33QRpost}
    \begin{align}
        \hat{\phi}_k&=w_1(\hat{\phi}_{k-1}-n-1)+n+1+w_2 k_w\label{33QRpost1}\\
        \hat{\psi}_k&=w_1(\hat{\psi}_{k-1}-m-1)+m+1+w_2 k_w\label{33QRpost2}\\
        \hat{\boldsymbol{\Phi}}_k&=w_1\hat{\boldsymbol{\Phi}}_{k-1}+w_2\sum_{j=1}^{k_w}\tilde{\boldsymbol{\Phi}}_j\label{33QRpost3}\\
        \hat{\boldsymbol{\Psi}}_k&=w_1\hat{\boldsymbol{\Psi}}_{k-1}+w_2\sum_{j=1}^{k_w}\tilde{\boldsymbol{\Psi}}_j\label{33QRpost4}
    \end{align}
\end{subequations}
where $w_1,w_2\in[0,1]$ are weights to balance the features extracted during the previous and current sliding windows.\par
To further suppress potential divergence led by malfunctioning sensors, we apply an additional weight to the updating process of observation noise covariance.
The sum of its auxiliary matrices is calculated in iterative form:
\begin{equation}
    \label{33Riter}
    \sum_{j=1}^{k_j}\tilde{\boldsymbol{\Psi}}_j=w_3\left( \sum_{j=1}^{k_j-1}\tilde{\boldsymbol{\Psi}}_j+\tilde{\boldsymbol{\Psi}}_{k_j} \right)
\end{equation}
where $k_j=1,2,\cdots,k_w$ and $w_3\in(0,1)$. \par
To find proper weights, we need to evaluate the error during the calculation process. Assuming the initial state is biased as $\hat{\boldsymbol{x}}_{f,0}=\bar{\boldsymbol{x}}_0+\delta\boldsymbol{x}_0$, and the final state is $\hat{\boldsymbol{x}}_{f,k_w}=\bar{\boldsymbol{x}}_{k_w}+\delta\boldsymbol{x}_{k_w}$. By iteratively conducting the augmented Kalman filter, the error propagation relation is obtained:
\begin{equation}
        \label{33errpro}
        \delta\boldsymbol{x}_{k_w}=\left(\prod_{j=1}^{k_w}(\boldsymbol{I}-\boldsymbol{K}_j\tilde{\boldsymbol{C}}_j)\boldsymbol{A}_{j-1}\right)\delta\boldsymbol{x}_0=\boldsymbol{E}\delta\boldsymbol{x}_0.
\end{equation}
Hence, we exploit the error propagation matrix $\boldsymbol{E}$ in two ways:
\vspace{-15pt}
\begin{subequations}
    \label{33exploitE}
    \begin{align}
        \text{average trace: }&\bar{\lambda} = \text{tr}(\boldsymbol{E})/6 \label{33exploitE1}\\
        \text{reduced determinant: }&\varrho=\sqrt[6]{|\boldsymbol{E}|} \label{33exploitE2}
    \end{align} 
\end{subequations}
where $\text{tr}(\cdot)$ is the trace of square matrix $(\cdot)$.\par
For $w_1$ and $w_2$, if $\bar{\lambda}$ is larger than a threshold $\lambda_0\in(0,1]$, this means the estimation is too bad to be used in covariance adjustment.
And we set $w_1=1$ and $w_2=0$, so that $\hat{\boldsymbol{Q}}_k$ and $\hat{\boldsymbol{R}}_k$ remain unchanged.
However, if $\bar{\lambda} < \lambda_0$, we can adjust covariance matrices, and the weights are calculated as $w_1=1-f_1\bar{\lambda}$ and $w_2=1-f_1+f_1\bar{\lambda}$, where $f_1\in(0,1)$ is a factor to control how fast the weights change.
As for $w_3$, we set $w_3=f_2+\varrho/f_2$, where $f_2\in(0,1)$ is another factor to control change speed.
This is the error propagation restriction, whose necessity and effectiveness will be thoroughly discussed in Subsection~\ref{secRes}.
In this way, the deserved posterior PDFs are obtained.
The posterior noise covariance matrices are calculated as mathematical expectations:
\begin{subequations}
    \label{33QRmean}
    \begin{align}
        \boldsymbol{Q}_k&=\mathbb{E}[\hat{\boldsymbol{Q}}_k]=(\hat{\phi}_k-n-1)^{-1}\hat{\boldsymbol{\Phi}}_k\label{33QRmean1}\\
        \bar{\boldsymbol{R}}_k&=\mathbb{E}[\hat{\boldsymbol{R}}_k]=(\hat{\psi}_k-m-1)^{-1}\hat{\boldsymbol{\Psi}}_k.\label{33QRmean2}
    \end{align}
\end{subequations}
\par
In practice, the message from the sensor may be lost, or the data quality may be poor sometimes, whereas the adaptive adjustment cannot catch up with the fast change. To handle this, a boolean amendment is included to represent the work condition of sensors:
\begin{equation}
    \label{33bool}
    \boldsymbol{S}_{k,i}=\begin{cases}
        \boldsymbol{I}_s,&\text{the $i$-th sensor functions normally}\\
        \varepsilon\boldsymbol{I}_s,&\text{the $i$-th sensor malfunctions}
    \end{cases}
\end{equation}
where $s$ is the dimension of the output from the $i$-th sensor, and $\varepsilon$ is a very large number.
We have $\boldsymbol{S}_{k,1}\in\mathbb{R}$ and $\boldsymbol{S}_{k,2}\in\mathbb{R}^{3\times 3}$ represents the work condition of the UWB and OF sensor respectively.
The noise covariance matrices for practical application are finally obtained by combining \eqref{33QRmean2} and \eqref{33bool} together: $\boldsymbol{R}_k=\boldsymbol{S}_k\bar{\boldsymbol{R}}_k\boldsymbol{S}_k$, where $\boldsymbol{S}_k=\mathrm{diag}(\boldsymbol{S}_{k,1},\boldsymbol{S}_{k,2})$.
The algorithm of inverse-Wishart smoother in each sliding window is illustrated in Algorithm~\ref{algIWS}.
\begin{algorithm}[htbp]
    \caption{Restricted Inverse-Wishart Smoother at Timestep $k$ of $\mathcal{O}(k_w(n^3+n^2m+m^2))$}
    \label{algIWS}
    \KwIn{
    $\check{\boldsymbol{x}}_{k-k_w:k-1}$, $\check{\boldsymbol{P}}_{k-k_w:k-1}$, $\boldsymbol{A}_{k-k_w:k-1}$, $\boldsymbol{u}_{k-k_w+1:k}$, $\boldsymbol{y}_{k-k_w+1:k}$, $\boldsymbol{C}_{k-k_w+1:k}$, $\hat{\phi}_{k-1}$, $\hat{\boldsymbol{\Phi}}_{k-1}$, 
    $\hat{\psi}_{k-1}$, $\hat{\boldsymbol{\Psi}}_{k-1}$, $\boldsymbol{E}$, $\boldsymbol{S}_k$
    }
    \KwParam{
    $k_w$, $\lambda_0$, $f_1$, $f_2$, $\varepsilon$, $n=6$, $m=4$
    }
    \KwOut{
    $\hat{\boldsymbol{x}}_{0:k_w}$, $\hat{\boldsymbol{P}}_{0:k_w}$, $\hat{\phi}_{k}$, $\hat{\boldsymbol{\Phi}}_{k}$, 
    $\hat{\psi}_{k}$, $\hat{\boldsymbol{\Psi}}_{k}$, $\boldsymbol{Q}_k$, $\boldsymbol{R}_k$
    }
    $\bar{\lambda}\gets \frac{1}{6}\text{tr}(\boldsymbol{E})$ and $\varrho\gets\sqrt[6]{|\boldsymbol{E}|}$\;
    \eIf{$\bar{\lambda}\geq\lambda_0$}{
        $w_1\gets1$\ and $w_2\gets0$\;
    }{
        $w_1\gets1-f_1\bar{\lambda}$ and $w_2\gets1-f_1+f_1\bar{\lambda}$\;
    }
    $w_3\gets f_2+\varrho/f_2$\;
    $\hat{\boldsymbol{P}}_{k_w}\gets\hat{\boldsymbol{P}}_{f,k_w}$ and $\hat{\boldsymbol{x}}_{k_w}\gets\hat{\boldsymbol{x}}_{f,k_w}$\;
    \For{$j=k:-1:k-k_w+1$}{
        Conduct backward smoother iteration by \eqref{33backward}\;
        Calculate auxiliary matrices by \eqref{33QRauxiliary} and \eqref{33Riter}\;
    }
    Calculate shape and DoF parameters by \eqref{33QRpost}\;
    Calculate $\boldsymbol{Q}$ and $\bar{\boldsymbol{R}}$ by \eqref{33QRmean}\;
    $\boldsymbol{R}_k\gets\boldsymbol{S}_k\bar{\boldsymbol{R}}_k\boldsymbol{S}_k$\;
\end{algorithm}

\subsection{Aerial Drag Estimator}\label{secMUE}
As is discussed in the introduction, the aerial drag effect is essential, but its estimation is high-cost due to the complexity of the UAV's aerodynamics. 
Thanks to the proposed sensor system, we have enough data to indirectly evaluate the aerial drag effect. 
We develop a novel aerial drag estimator, which omits intricate aerodynamics and spares computational resources.
First, consider the velocity part of the prior state obtained by dynamics \eqref{31motionModel}: $\check{\boldsymbol{v}}_j=(\boldsymbol{I}_3-\mathrm{dt}\cdot\boldsymbol{\mu})\hat{\boldsymbol{v}}_{j-1}+\mathrm{dt}\cdot\boldsymbol{i}_k$, where $\check{\boldsymbol{v}}_j$ and $\hat{\boldsymbol{v}}_{j-1}$ are velocity part of $\check{\boldsymbol{x}}_j$ and $\hat{\boldsymbol{x}}_{j-1}$ respectively.
Also, with measurements from the sensor system, we obtain posterior $\hat{\boldsymbol{v}}_{j}$.
Our goal is to adjust $\boldsymbol{\mu}$ such that $\check{\boldsymbol{v}}_j$ approaches $\hat{\boldsymbol{v}}_{j}$, and hence we can formulate a cost function in quadratic form w.r.t. $\boldsymbol{\mu}$: $J_j=||\hat{\boldsymbol{v}}_{j}-\check{\boldsymbol{v}}_j||_2^2=(\hat{\boldsymbol{v}}_{j}-\check{\boldsymbol{v}}_j)^\mathrm{T}(\hat{\boldsymbol{v}}_{j}-\check{\boldsymbol{v}}_j)$.
By taking the derivative:
\begin{equation}
    \label{34derivative}
    \partial J_j/\partial \boldsymbol{\mu}=2\mathrm{dt}\left(\hat{\boldsymbol{v}}_j-(\boldsymbol{I}_3-\mathrm{dt}\boldsymbol{\mu})\hat{\boldsymbol{v}}_{j-1}-\mathrm{dt}\boldsymbol{i}_j\right)\hat{\boldsymbol{v}}_{j-1}^\mathrm{T},
\end{equation}
we can conduct a gradient descent method to amend $\boldsymbol{\mu}$:
\begin{equation}
    \label{34gradDown}
    \boldsymbol{\mu}\gets \boldsymbol{\mu}-\ell_k\cdot\partial J_j/\partial \boldsymbol{\mu}
\end{equation}
where $\ell_k$ is step length.\par
Since choosing the proper step length is a critical problem in gradient descent, we design a strategy to accomplish this.
The noise covariance matrices inherently reflect the deviation of dynamics and observation model w.r.t. actual situation. 
Hence, they are natural indicators to evaluate the performance of both models.
Under this premise, we calculate $\ell_k$ as follows:
\begin{equation}
    \label{34step}
    \ell_k=
        b_u-\frac{(b_u-b_l)\sqrt[4]{|\boldsymbol{R}_k|}}{\sqrt[6]{|\boldsymbol{Q}_k|}},\;\text{when}\; \sqrt[6]{|\boldsymbol{Q}_k|}>\sqrt[4]{|\boldsymbol{R}_k|}
\end{equation}
where $b_u$ and $b_l$ are upper and lower bounds for step length.
Intuitively, if $\sqrt[6]{|\boldsymbol{Q}_k|}\leq\sqrt[4]{|\boldsymbol{R}_k|}$, the sensors are less reliable than the dynamics. Hence, we should not amend $\boldsymbol{\mu}$ with measurements and set $\ell_k=0$.
On the other hand, once $\sqrt[6]{|\boldsymbol{Q}_k|}>\sqrt[4]{|\boldsymbol{R}_k|}$, we can start our adjustment.
And as $\sqrt[4]{|\boldsymbol{R}_k|}$ gets smaller, $\ell$ grows from its lower bound to upper bound. \par
Finally, we compress all three parts\textemdash augmented Kalman filter, restricted inverse-Wishart smoother and aerial drag estimator, into a single compact algorithm, restricted adaptive sliding window estimator (RASWE).
The Algorithm~\ref{algSATE} is the pseudo-code version of the proposed method RASWE.
\begin{algorithm}[htbp]
    \caption{Restricted Adaptive Sliding Window Estimator at Timestep $k$ of $\mathcal{O}(k_w(n^3+n^2m+nm^2+m^3))$}
    \label{algSATE}
    \KwIn{
    $\check{\boldsymbol{x}}_{k-k_w:k-1}$, $\check{\boldsymbol{P}}_{k-k_w:k-1}$, $\boldsymbol{\mu}$, $\mathrm{dt}$, $\boldsymbol{u}_{k-k_w+1:k}$, $\boldsymbol{y}_{k-k_w+1:k}$, $\hat{\phi}_{k-1}$, $\hat{\boldsymbol{\Phi}}_{k-1}$, 
    $\hat{\psi}_{k-1}$, $\hat{\boldsymbol{\Psi}}_{k-1}$, $\boldsymbol{S}_k$
    }
    \KwParam{
    $k_w$, $\lambda_0$, $f_1$, $f_2$, $\varepsilon$, $b_l$, $b_u$, $n=6$, $m=4$
    }
    \KwOut{
    $\hat{\boldsymbol{x}}_{k-k_w:k}$, $\hat{\phi}_{k}$, $\hat{\boldsymbol{\Phi}}_{k}$, 
    $\hat{\psi}_{k}$, $\hat{\boldsymbol{\Psi}}_{k}$, $\boldsymbol{\mu}$
    }
    Calculate $\boldsymbol{Q}_k$ and $\bar{\boldsymbol{R}}_k$ by \eqref{33QRmean}\;
    $\boldsymbol{R}_k\gets\boldsymbol{S}_k\bar{\boldsymbol{R}}_k\boldsymbol{S}_k$\;
    Calculate $\boldsymbol{A}_{k-k_w:k-1}$ by \eqref{31motionModel}\;
    Calculate $\boldsymbol{C}_{k-k_w+1:k}$ by \eqref{31approP} and \eqref{31observationModel}\;
    Augment $\boldsymbol{y}_{k-k_w+1:k}$, $\boldsymbol{C}_{k-k_w+1:k}$, $\boldsymbol{R}_k$ by \eqref{32yCR}\;
    $\hat{\boldsymbol{P}}_{f,0}\gets\check{\boldsymbol{P}}_0$ and $\hat{\boldsymbol{x}}_{f,0}\gets\check{\boldsymbol{x}}_0$\;
    \For{$j=k-k_w+1:+1:k$}{
    Conduct Kalman filter by \eqref{32KFforward}\;
    }
    Calculate $\boldsymbol{E}$ by \eqref{33errpro}\;
    Run Algorithm~\ref{algIWS}\;
    Calculate step length by \eqref{34step}\;
    \For{$j=k-k_w+1:+1:k$}{
    Calculate derivative by \eqref{34derivative}\;
    Update $\boldsymbol{\mu}$ by \eqref{34gradDown}\;
    }
\end{algorithm}

\subsection{Observability Analysis}\label{secAna}
The observability matrix of sensor system modeled by \eqref{31motionModel} and \eqref{31observationModel} is:
\begin{equation}
    \label{41obsMAT}
    \boldsymbol{\mathcal{O}}_k=\begin{bmatrix}
        \boldsymbol{C}_k^\mathrm{T} & \boldsymbol{A}_{k-1}^\mathrm{T}\boldsymbol{C}_k^\mathrm{T} & \cdots & \boldsymbol{A}_{k-1}^{5\mathrm{T}}\boldsymbol{C}_k^\mathrm{T}
    \end{bmatrix}^\mathrm{T}
\end{equation}
where $\boldsymbol{A}^{n\mathrm{T}}=(\boldsymbol{A}^n)^\mathrm{T}=(\boldsymbol{A}^\mathrm{T})^n$ for brevity. We introduce the first theorem to guarantee the observability.
\begin{theorem}\label{the1OBS}
    If $\boldsymbol{p}_k\neq\boldsymbol{0}$ and $\mathrm{dt}\cdot\boldsymbol{\mu}\neq\boldsymbol{I}$, then $\mathrm{rank}(\boldsymbol{\mathcal{O}}_k)=6$ and this system is observable.
\end{theorem}
\begin{proof}
    We select a submatrix of $\boldsymbol{\mathcal{O}}$ that has the form:
    \begin{equation}
        \label{41OBSmat}
        \tilde{\boldsymbol{\mathcal{O}}}=\frac{\boldsymbol{I}_6\otimes \tilde{\boldsymbol{p}}^\mathrm{T}}{||\tilde{\boldsymbol{p}}||_2}\begin{bmatrix}
             \boldsymbol{I}_3 & \mathrm{dt}\cdot\boldsymbol{X}_0\\
             \vdots & \vdots\\
             \boldsymbol{I}_3 & \mathrm{dt}\cdot\boldsymbol{X}_5\\
        \end{bmatrix}\in\mathbb{R}^{6\times6}
    \end{equation}
    where the label $k$ is omitted without ambiguity.
    Iteratively, $\boldsymbol{X}_{i+1}=\boldsymbol{I}_3+\boldsymbol{X}_i(\boldsymbol{I}_3-\mathrm{dt}\boldsymbol{\mu})$ with $\boldsymbol{X}_0=\boldsymbol{0}_3$.
    One can verify that each $\boldsymbol{X}_i$ is independent of the others if $\mathrm{dt}\cdot\boldsymbol{\mu}\neq\boldsymbol{I}$.
    Hence every row of $\tilde{\boldsymbol{\mathcal{O}}}$ is independent of one another, which means $\mathrm{rank}(\tilde{\boldsymbol{\mathcal{O}}})=6$.
    Also, $\mathrm{rank}(\tilde{\boldsymbol{\mathcal{O}}})\leq\mathrm{rank}(\boldsymbol{\mathcal{O}}_k)\leq6$ because of submatrix's property.
    Hence, $\mathrm{rank}(\boldsymbol{\mathcal{O}}_k)=6$.
\end{proof}\par
The two conditions of Theorem~\ref {the1OBS} make sense.
If $\boldsymbol{p}_k\approx\boldsymbol{0}$, the UAV is almost at the same position as UWB; hence UWB fails.
However, in practical application, the UAV hardly hovers near the anchor but usually flies far away.
On the other hand, in common aerial surroundings, the drag effect is negligible and $\mathrm{dt}\cdot\boldsymbol{\mu}\ll\boldsymbol{I}$.
Furthermore, even if one of the conditions fails, this will not last for a long time, and the sliding window can handle this by evaluating a long trajectory, which is observable almost everywhere.
This introduces the second theorem.\par
\begin{theorem}\label{the2obs}
    The observability of the augmented model in \eqref{32KFforward} is guaranteed without any additional condition.
\end{theorem}
\begin{proof}
    Since $\boldsymbol{C}_k$ is augmented with an identity matrix, we can directly select a submatrix $\boldsymbol{I}_6$ from the observability matrix.
    Hence, the augmented model is observable.
\end{proof}\par
The Theorem~\ref{the2obs} tells us that when observation at one timestep fails, we can directly leverage the previous estimation.
Only if the sensor degradation lasts longer than the sliding window length will the observability lose.
Notably, at the timestep $k_w$, there is no augmentation, and the observability relies on Theorem~\ref{the1OBS}.
\section{Experiments}\label{secExp}
\subsection{Experimental Setup}
In all experiments, a VICON\textsuperscript{\textregistered} motion capture system records ground truth at 200Hz.
A total of $14$ VICON cameras are engaged with sufficiently large overlapping perceptive areas to capture all trajectories.
The quadrotor unmanned aerial vehicle is steered under the control of the open-source Pixhawk\textsuperscript{\textregistered}.
The onboard computer is an NVIDIA\textsuperscript{\textregistered} NX computing mounted Intel Atom x7 (four cores, 1.8GHz).
The IMU (module CHCNAV CL-510) is utilized for acceleration measurements.
The Nooploop\textsuperscript{\textregistered} UWB (module LinkTrack LTPS\textsuperscript{\textregistered}) is selected for range measurement cooperated with a fixed anchor also equipped with the same UWB.
Typically, the measurement error of UWB is calibrated to a centimeter-level ($8.59$cm RMSE) but may degrade to a decimeter-level ($25.66$cm RMSE) in harsh environments due to interference.
The operation area of single anchor UWB is a ball with a decameter-level radius far larger than the room size.
Hence, the overlapping area of UWB and other sensors is sufficient to cover the whole testbed.
The optical flow sensor (module NiMing v4) is adopted for velocity measurements.
The communication system between sensors and UAV is built based on the Robot Operating System (ROS).
The IMU, UWB, and OF measurements are collected at 25Hz.\par
We conduct experiments with different velocities in common laboratory surroundings and record three datasets $\text{B}_1$ to $\text{B}_3$.
To evaluate our method in harsh environments, we repeat experiments with blinking light and a smoke generator, and record four datasets $\text{Y}_0$ to $\text{Y}_3$.
Since this section mainly focuses on state estimation performance, the UAV is controlled manually in all sever datasets.
To overcome jagged edges in data due to the simplified controller and poor physical structure lacking vibration damping, we apply a Savitzky–Golay filter with polynomial order 3 and frame length 9 to the estimated position.
The configuration of the testbed is illustrated in Fig.~\ref{fig:testbed}.
The obstacles and smoke shown in Subfigure (b) simulate an office or factory environment, while the blinking light mimics a streetlight.
The constructed environment provides a certain level of diversity and realism, offering approximation to the challenges encountered in practical scenarios.\par
The experiment parameters are displayed in Table~\ref{tab:param}, which are chosen via trial and error.
In particular, the factors, $f_1$ and $f_2$, should be small so that their corresponding weights approach $1$.
The bounds for step length should also be small to alleviate numerical oscillation.
As for the initial values on the right column, for example, we first empirically set $\boldsymbol{\mu}_0=0.1\cdot\boldsymbol{I}_3$ and find it approximately converges to the value in Table~\ref{tab:param}.
To validate covariance estimation, we set covariance parameters to be multiples of the identity matrix.
In \eqref{33bool}, if the ROS message from UWB is unfortunately empty, we judge that UWB is malfunctioning.
The OF unit directly returns a quality metric (integral from $0$ to $255$), and we assert the OF unit is malfunctioning in \eqref{33bool} if the quality is below $255$.\par
\begin{figure}[htbp]
    \centering
    \vspace{-7pt}
    \includegraphics[width=0.98\linewidth]{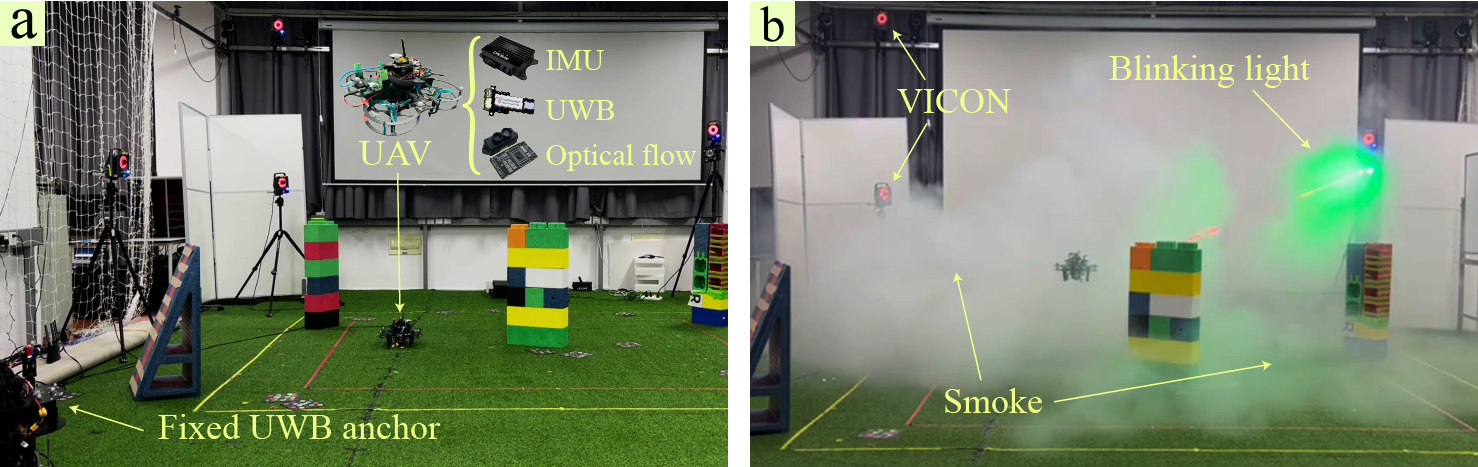}
    \vspace{-2pt}
    \caption{Testbeds of experiments. Subfigure (a) shows common surroundings; subfigure (b) is a smoky and blinking environment.}
    \label{fig:testbed}
\end{figure}
\begin{table}[htbp]
    \centering
    \setlength{\tabcolsep}{4.8mm}
    \caption{Parameters and Initial Values of Estimator in Experiments}
    \begin{tabularx}{0.98\linewidth}{clcl}
        \toprule
        \textbf{Item} & \textbf{Value} & \textbf{Item} & \textbf{Initial Value}\\
        \midrule
        $k_w$ & $10$                    & $\boldsymbol{\mu}_0$ & $\text{diag}(0.2,0.2,0.8)$\\
        $\lambda_0$ & $1\times 10^{-3}$ & $\check{\boldsymbol{P}}_0$ & $0.1\cdot\boldsymbol{I}_6$\\
        $f_1$ & $1\times 10^{-2}$       & $\check{\boldsymbol{x}}_0$ & VICON ground truth\\
        $f_2$ & $0.1$                   & $\hat{\boldsymbol{\Phi}}_0$ & $17\cdot\boldsymbol{I}_6$\\
        $b_u$ & $1\times 10^{-2}$       & $\hat{\phi}_0$ & $10$\\
        $b_l$ & $1\times 10^{-3}$       & $\hat{\boldsymbol{\Psi}}_0$ & $13\cdot\boldsymbol{I}_4$\\
        $\varepsilon$ & $1\times 10^{3}$& $\hat{\psi}_0$ & $8$\\
        \bottomrule
    \end{tabularx}
    \label{tab:param}
\end{table}
\vspace{-10pt}
\subsection{Evaluation of Overall Performance}
\vspace{-2pt}
To assess the effectiveness and stability of the proposed method, we compare it to the DWE in \cite{10947092}, which is equivalent to the combination of augmented Kalman filter \eqref{32KFforward} and backward smoother \eqref{33backward} without adjusting parameters like covariance and aerial drag matrices. 
The results are displayed in Table~\ref{tab:overall}, where we evaluate the root mean square error (RMSE) and the standard deviation (STD) along three axes.\par
\begin{table}[htbp]
    \centering
    \vspace{-10pt}
    \setlength{\tabcolsep}{2.52mm}
    \caption{Overall Performance Comparison between DWE and RASWE}
    \begin{tabularx}{0.98\linewidth}{cccccccc}
        \toprule
        \multirow{2}{*}{\makecell[c]{\textbf{Data}\\ \textbf{Set}}} & \multirow{2}{*}{\textbf{Method}} & \multicolumn{3}{c}{\textbf{RMSE} (meter)} & \multicolumn{3}{c}{\textbf{STD} (meter)} \\ \cline{3-8}
        && $x$ & $y$ & $z$ & $x$ & $y$ & $z$ \\
        \midrule
        \multirow{2}{*}{$\text{B}_1$} & DWE & 0.07 & 0.14 & 0.05 & 0.04 & 0.07 & 0.03 \\
        & RASWE & 0.07 & 0.14 & 0.03 & 0.04 & 0.09 & 0.02 \\
        \midrule
        \multirow{2}{*}{$\text{B}_2$} & DWE & 0.08 & 0.13 & 0.06 & 0.05 & 0.08 & 0.02 \\
        & RASWE & 0.09 & 0.11 & 0.05 & 0.06 & 0.07 & 0.02 \\
        \midrule
        \multirow{2}{*}{$\text{B}_3$} & DWE & 0.24 & 0.38 & 0.11 & 0.15 & 0.19 & 0.06 \\
        & RASWE & 0.14 & 0.11 & 0.05 & 0.10 & 0.07 & 0.03 \\
        \midrule
        \multirow{2}{*}{$\text{Y}_0$} & DWE & 0.32 & 0.62 & 0.09 & 0.20 & 0.38 & 0.05 \\ 
        & RASWE & 0.25 & 0.25 & 0.05 & 0.15 & 0.17 & 0.03 \\
        \midrule
        \multirow{2}{*}{$\text{Y}_1$} & DWE & 0.19 & 0.43 & 0.06 & 0.11 & 0.26 & 0.03 \\
        & RASWE & 0.23 & 0.22 & 0.08 & 0.13 & 0.14 & 0.06 \\
        \midrule
        \multirow{2}{*}{$\text{Y}_2$} & DWE & 0.17 & 0.35 & 0.10 & 0.10 & 0.26 & 0.08 \\
        & RASWE & 0.22 & 0.19 & 0.15 & 0.11 & 0.13 & 0.12 \\
        \midrule
        \multirow{2}{*}{$\text{Y}_3$} & DWE & 0.39 & 0.43 & 0.24 & 0.23 & 0.27 & 0.22 \\
        & RASWE & 0.37 & 0.38 & 0.06 & 0.25 & 0.25 & 0.04 \\
        \bottomrule
    \end{tabularx}
    \label{tab:overall}
\end{table}
From the results, it proves the proposed method RASWE outperforms DWE, especially in harsh environments, $\text{Y}_0$ to $\text{Y}_3$.
In common environments, however, RASWE may not always show its advantage since the noise covariance matrices and aerial drag matrix seldom change much.
The typical estimation results of dataset $\text{B}_2$ is illustrated in Fig.~\ref{fig:bag02xyz}.
Notably, RASWE reduces the error along the x-axis and y-axis to the same degree, whereas DWE always leaves one axis worse.
This is because noise covariance matrices are natural weights, and RASWE adjusts them accordingly to balance the error.
But without adaptive adjustment, the weights of two axes are empirically set to be the same, which conflicts with the fact and causes performance degradation of DWE. \par
\begin{figure}[htbp]
    \centering
    \vspace{-7pt}
    \includegraphics[width=0.98\linewidth]{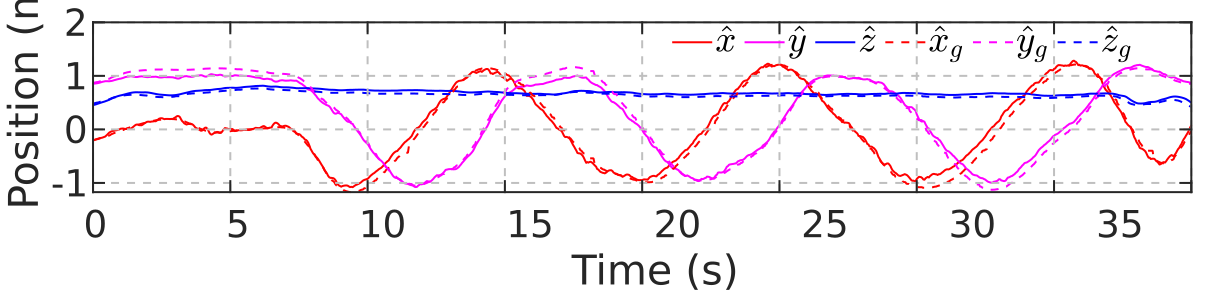}
    \vspace{-3pt}
    \caption{Typical position estimation results of RASWE on dataset $\text{B}_2$, where dotted line and label "$g$" are used to denote ground truth from VICON.}
    \label{fig:bag02xyz}
\end{figure}
As illustrated in subfigure (a) of Fig.~\ref{fig:bag02qr}, for the typical experiment $\text{B}_2$, the diagonal entry of the process noise covariance matrix corresponding to the y-axis is more significant than that corresponding to the x-axis, which is consistent with DWE result that RMSE along the y-axis is more significant than that along the x-axis.
Although the symmetry of aerodynamics indicates the weights corresponding to the x-axis and y-axis should be the same, the measurements from UWB cause this difference.
It only provides distance information, a weak constraint of only one dimension.
The limited dimension of information leads to the fact that, in practice, estimation along one axis usually performs worse.
On the other hand, it is the OF unit that plays a part in velocity estimation, while UWB does not affect this process.
Hence, we see $Q_{44}=Q_{55}$, which agrees with the symmetry of aerodynamics.
Thanks to our algorithm, the process noise covariance matrix becomes competent in revealing these points, which helps the estimator find more proper weights.\par
For observation noise covariance matrix in subfigure (b), we evaluate its shape parameter here instead of the covariance matrix itself since its change is slight and uneasy for intuitive understanding ought to error propagation restriction \eqref{33Riter}.
It shows that the entry corresponding to UWB is the largest one because this range of odometry is not always sufficiently accurate and can only provide a global reference.
As for the OF sensor, entries of the x-axis and y-axis are the same due to symmetry, and the entry of the z-axis is slightly larger than theirs since vertical velocity obtained by the difference method is less precise than horizontal velocity obtained by the OF method. 
In this way, RASWE makes full use of information from sensors and tries its best to avoid divergence led by malfunctioning sensors.\par
\begin{figure}[htbp]
    \centering
    \vspace{-4pt}
    \includegraphics[width=0.98\linewidth]{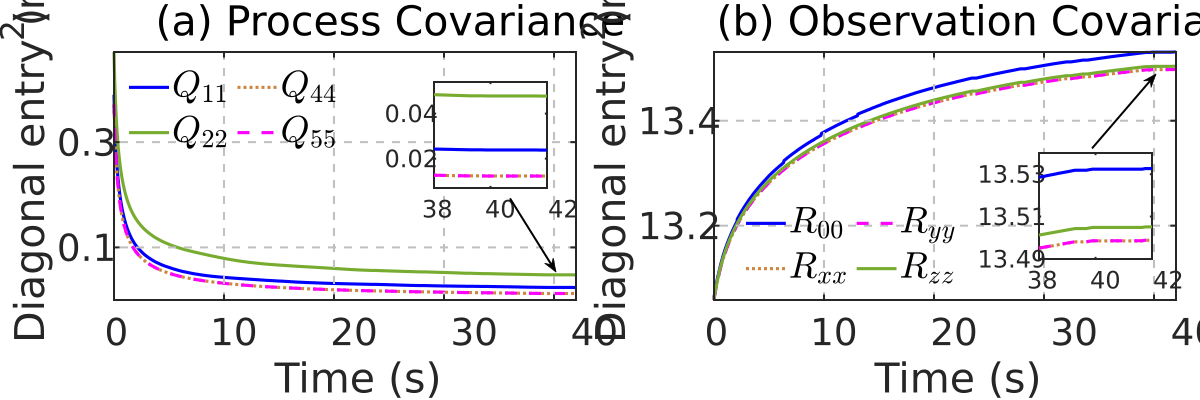}
    \vspace{-4pt}
    \caption{
    Diagonal entry changing tendency of noise covariance matrices (dataset: $\text{B}_2$).
    Subfigure (a) shows entries of process noise covariance matrices where $Q_{11}$, $Q_{22}$, $Q_{44}$ and $Q_{55}$ correspond to x-position, y-position, x-velocity and y-velocity respectively;
    subfigure (b) shows entries of the shape parameter of observation noise covariance where $R_{00}$ corresponds to UWB while the others correspond to three axes of the optical flow sensor.
    }
    \label{fig:bag02qr}
\end{figure}
We further compare RASWE with MST-SWVAKF \cite{9944196}, which also leverages inverse-Wishart distribution to estimate covariance matrices. Through a trial-and-error method, we set its parameters, $\beta=\delta=0.98$, $\varphi_i=1$, $\rho_i=1-e^{-4}$,$\kappa_{\text{upper}}=0.61$, $\kappa_{\text{lower}}=-2.89$, $L_{\text{max}}=9$, $L_{\text{min}}=4$, and other covariance parameters are kept the same as those in Table~\ref{tab:param}.
During experiments, we find MST-SWVAKF with fixed window length (abbr. "FWL") $L\equiv L_{\text{max}}$ performs better than adaptive window length.
The results are illustrated in Fig.~\ref{fig:spider}, where we evaluate three axes together using RMSE of Euclidean distance between estimated position and ground truth, $\sqrt{\frac{1}{n}\sum_{i=1}^{n}||\boldsymbol{p}_i-\boldsymbol{p}_{i,g}||^2_2}$, instead of examining three separated axes.
It shows that MST-SWVAKF with FWL has competitive results on some datasets but suffers a significant accuracy loss on others, especially in common environments $\text{B}_1$ to $\text{B}_3$.
The original one is even more likely to be misled by noise since frequently changed window length can greatly introduce numeric instability.
These results support the necessity and effectiveness of applying additional restrictions to RASWE, which will be discussed thoroughly with ablation experiments in the following subsection.
\begin{figure}[htbp]
    \centering
    \vspace{-4pt}
    \includegraphics[width=0.98\linewidth]{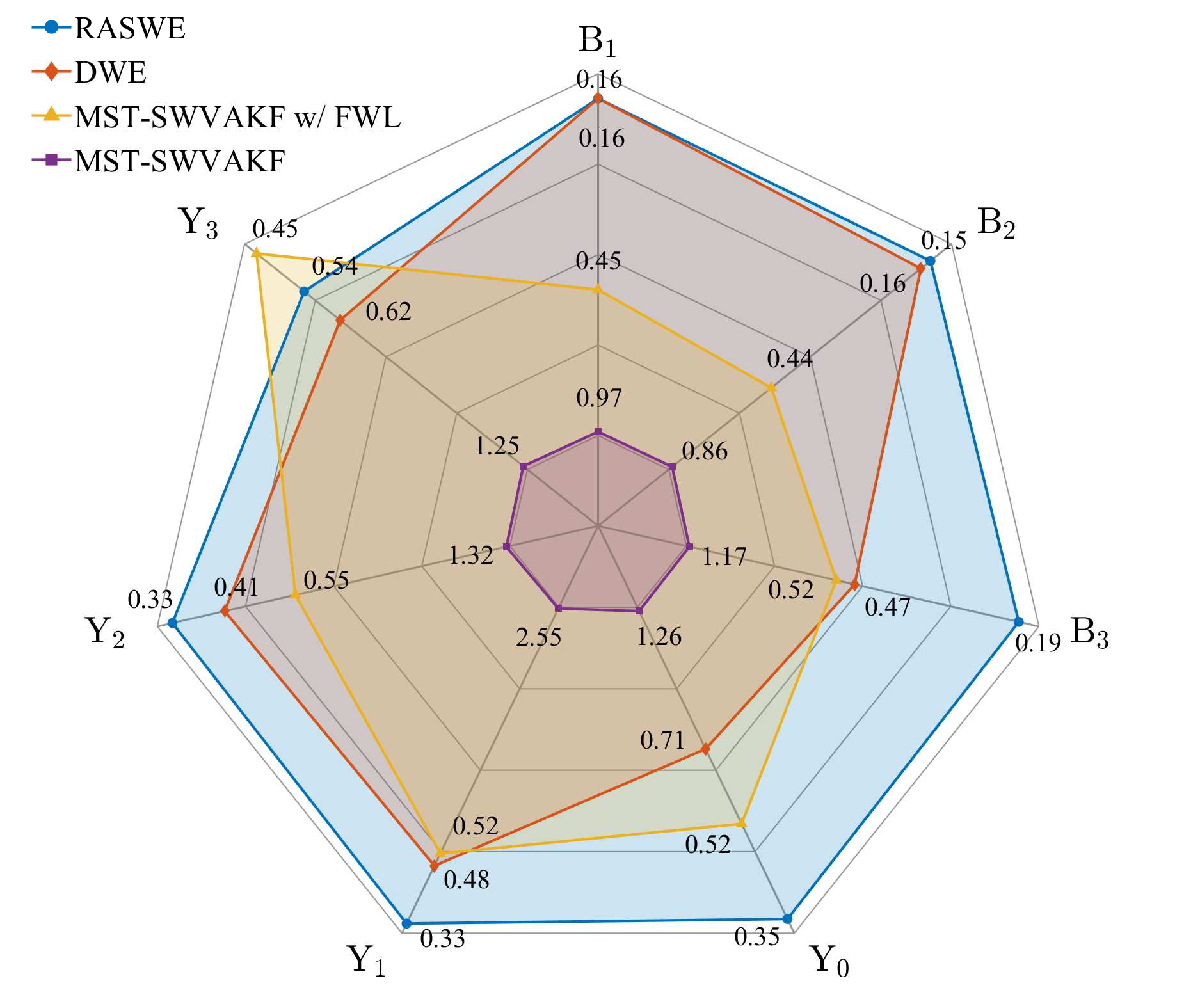}
    \vspace{-9pt}
    \caption{Overall performance comparison between RASWE, DWE and MST-SWVAKF on root mean square error of all datasets (unit: meter). Better performance corresponds to points plotted farther from the center; thus, a larger enclosed area indicates overall superior results.}
    \label{fig:spider}
\end{figure}
\subsection{Validation of Inherent Restrictions}\label{secRes}
As mentioned before, we apply several inherent restrictions to overcome sensor malfunction and suppress potential divergence.
In this subsection, we delve into the effects of them.
The metric here is the same RMSE as that in Fig.~\ref{fig:spider}.
All results are displayed in Table~\ref{tab:restriction}, where "none" denotes original RASWE the same as that in Table~\ref{tab:overall}.\par
\begin{table}[htbp]
    \centering
    \vspace{-10pt}
    \setlength{\tabcolsep}{2.5mm}
    \caption{RMSE comparison between Modified RASWEs with Certain Restriction Cancelled (Unit: meter)}
    \begin{tabularx}{0.98\linewidth}{cccccccc}
        \toprule
        \makecell[c]{\textbf{Cancelled}\\ \textbf{Restriction}} & $\textbf{B}_\textbf{1}$ & $\textbf{B}_\textbf{2}$ & $\textbf{B}_\textbf{3}$ & $\textbf{Y}_\textbf{0}$ & $\textbf{Y}_\textbf{1}$ & $\textbf{Y}_\textbf{2}$ & $\textbf{Y}_\textbf{3}$ \\
        \midrule
        None        & \textbf{0.16} & \textbf{0.15} & \textbf{0.19} & \textbf{0.35} & \textbf{0.33} & \textbf{0.33} & \textbf{0.54} \\
        ErrProp     & \textbf{0.16} & \textbf{0.15} & 0.20 & 0.45 & 0.45 & 0.43 & 0.59 \\
        coherence    & 1.27 & 1.04 & 1.49 & 1.50 & 1.44 & 1.55 & 1.78 \\
        Consist     & 0.94 & 1.04 & 4.04 & 1.57 & 2.28 & 1.23 & 1.83 \\
        \bottomrule
    \end{tabularx}
    \label{tab:restriction}
\end{table}
Among all restrictions, the most complicated one is error propagation restriction (abbr. ErrProp or EP), \eqref{33errpro} and \eqref{33exploitE}.
We cancel this restriction by setting three weights in \eqref{33QRpost} and \eqref{33Riter} to be 1, $w_1=w_2=w_3=1$.
From the results, we see that in common environments, whether to cancel it or not does not affect estimation precision much.
In harsh environments, however, the estimation deteriorates after the restriction is canceled.
This shows RASWE's advantage in handling harsh environments, as shown in Fig.~\ref{fig:yan0.6yc}.
Without error propagation restriction, the estimator can hardly track the position, especially when the UAV takes a sharp turn because it has been used for mild movement.\par
\begin{figure}[htbp]
    \centering
    \vspace{-4pt}
    \includegraphics[width=0.98\linewidth]{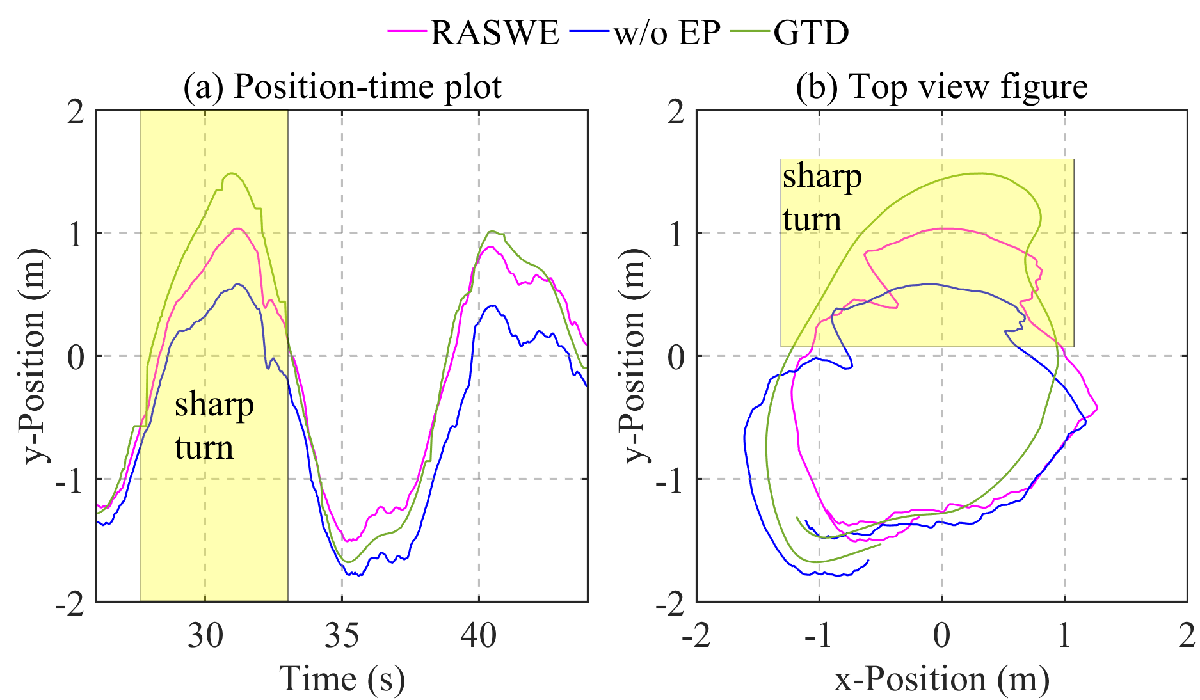}
    \vspace{-4pt}
    \caption{Comparison between estimations with and without error propagation restriction on the sharp turn of dataset $\text{Y}_2$, where "GTD" means ground truth data.}
    \label{fig:yan0.6yc}
\end{figure}
Another perspective is illustrated in Fig.~\ref{fig:yan0.4pvz}.
The estimation is initially interfered with by sudden rapid fluctuations, and RASWE gradually recovers from divergence. 
In contrast, the one without restriction is misled by malfunctioning sensors, leaving an estimation bias.
The root cause of position estimation divergence is the divergence of velocity estimation shown in subfigure (b).
The height laser of the OF unit mistakes smoke particles as ground and measures the wrong z-velocity, which misleads velocity estimation.
Thanks to error propagation restriction, velocity estimation converges again soon, which results in gradual convergence of position estimation through integral relation.\par
\begin{figure}[htbp]
    \centering
    \vspace{-4pt}
    \includegraphics[width=0.98\linewidth]{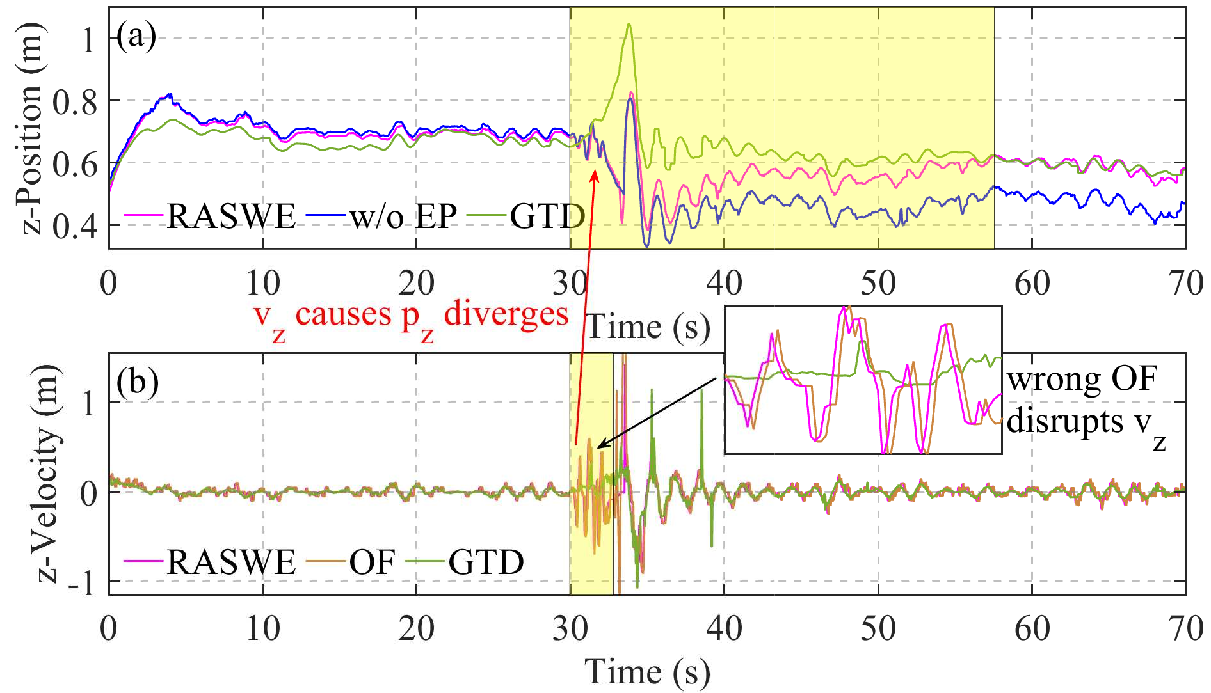}
    \vspace{-4pt}
    \caption{Position and velocity estimation of RASWE on dataset $\text{Y}_1$ with sudden rapid fluctuations, where "GTD" means ground truth data. Subfigure (a) is a comparison along the z-axis; subfigure (b) shows the reason for the divergence and recovery of position estimation.}
    \label{fig:yan0.4pvz}
\end{figure}
Aerial drag adjustment is also affected at the beginning of sudden fluctuations since the calculation of its gradient \eqref{34derivative} relies directly on velocity estimation.
The changing tendency of aerial drag adjustment on dataset $\text{Y}_1$ is illustrated in subfigure (a) of Fig.~\ref{fig:yan0.4mdpe}.
A sharp increasing curve emerges at first due to fluctuations, but soon later, it is replaced by a horizontal line.
This is because those wrong adaptive adjustments are stopped by error propagation restriction, and subfigure (b) shows how parameters respond to fluctuations.
Soon after rapid fluctuations, adjustments continue and converge again.
Notably, the details of the convergent line show slight oscillation at the micro-scale even if it seems horizontal in subfigure (a), which is different from the horizontal line resulting from error propagation restriction.\par
The error propagation matrix acutely perceives anomalous estimation.
As is shown in subfigure (b) of Fig.~\ref{fig:yan0.4mdpe}, the diagonal entry corresponding to z-velocity, $E_{66}$, responds to sudden fluctuations. 
However, the entry corresponding to the z-position, $E_{33}$, remains almost unchanged.
That is because position estimation is based on velocity estimation via integral, and the calculation process is normal, although the results are misleading. 
This shows the proficiency of $\boldsymbol{E}$ in \eqref{33errpro} to serve as an online error inspector.
The parameters in \eqref{33exploitE} are also displayed in subfigure (b), where the average trace (abbr. AveTrc) is more sensitive to slight interference, and the reduced determinant (abbr. RedDet) focuses on significant fluctuations more.
This proves effective in merging these two parameters in error propagation restriction.
\begin{figure}[htbp]
    \centering
    \vspace{-4pt}
    \includegraphics[width=0.98\linewidth]{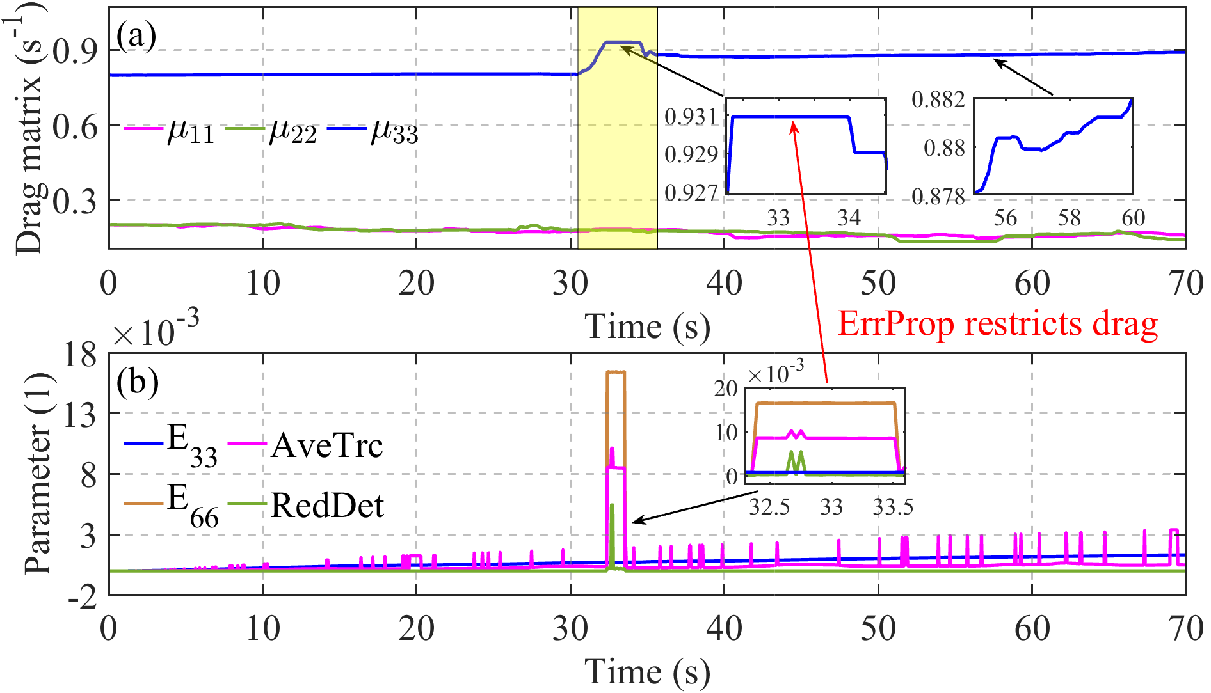}
    \vspace{-4pt}
    \caption{Error propagation restriction prevents aerial drag adjustment from divergence (dataset $\text{Y}_1$). Subfigure (a) illustrates the changing tendency of aerial drag coefficients.
    Subfigure (b) displays the changing tendency of parameters in \eqref{33exploitE} and diagonal entry of the error propagation matrix in \eqref{33errpro}.}
    \label{fig:yan0.4mdpe}
\end{figure}
The second restriction is coherence restriction, which exploits redundant information due to overlapping sliding windows, as illustrated in Fig.~\ref{fig:SW}.
We remove it by simply omitting the augmentation step for KF \eqref{32yCR1}--\eqref{32yCR3}.
Our estimator relies on coherence restriction to inherit historical information since the sliding window size is 10 in our experiments, which is typically too small to perform proper estimation.
Hence, as is shown in Table~\ref{tab:overall}, every estimation has a significant error if the coherence restriction is canceled.\par
To keep numerical stability, we apply one additional restriction called consistency restriction.
When RASWE begins in every new sliding window, we re-initialize $\Check{\boldsymbol{P}}_0=0.1\cdot\boldsymbol{I}_6$. We keep all $\boldsymbol{Q}$'s and $\tilde{\boldsymbol{R}}$'s the same for every timestep, which is obtained by Algorithm~\ref{algIWS} in the last sliding window.
To cancel this restriction, we use a matrix sequence of size $k_w+1$ to record historical covariance matrices $\boldsymbol{P}$'s, $\boldsymbol{Q}$'s and $\boldsymbol{R}$'s.
In this way, all the noise covariance matrices at each timestep differ.
Hence, the estimator becomes more numerically sensitive, increasing the divergence probability.
This counts for the results that diverge badly.
On the other hand, since the sliding window size is quite small, it is rational to assume all covariance matrices are approximately the same.\par
The conclusion is that coherence and consistency restrictions are indispensable for our estimator. 
Error propagation restriction is essential to handle harsh environments, significantly suppressing divergence.
\subsection{Verification of Online Estimation and Control}
To verify our localization algorithm in real-world control applications, we additionally conduct experiments in the same laboratory environment, with the controller in \cite{10947092} to steer the UAV instead of manual control.
The target trajectory is a circle of a 1-meter radius executed at varying speeds.
All control experiments employ the same parameter settings as those listed in Table~\ref{tab:param} for consistency.
However, to account for the potential increase in sensor noise that may occur during the autonomous take-off and landing procedures, the tolerance for sensor noise is relaxed and the malfunction detection threshold in \eqref{33bool} is accordingly raised.
Moreover, we directly leverage the height from the laser unit in the OF sensor to calibrate the z-position, rather than employing the differencing method to estimate z-velocity as done previously.
This approach mitigates the additional errors that can arise from numerical differentiation and avoids the phase lag introduced by filtering differentiation noise, both of which could adversely affect control performance.\par
\begin{table}[tbp]
    \centering
    \vspace{-10pt}
    \setlength{\tabcolsep}{2.35mm}
    \caption{Performance Evaluation of Control and Estimation of Different Control Speed}
    \begin{tabularx}{0.98\linewidth}{ccc}
        \toprule
        \textbf{Speed} (m/s) & \textbf{Estimation RMSE} (meter) & \textbf{Control RMSE} (meter)\\
        \midrule
        0.20 & 0.32 & 0.11 \\
        0.35 & 0.45 & 0.13 \\
        0.60 & 0.44 & 0.14 \\
        \bottomrule
    \end{tabularx}
    \label{tab:control}
\end{table}

\begin{figure}[htbp]
    \centering
    \includegraphics[width=0.98\linewidth]{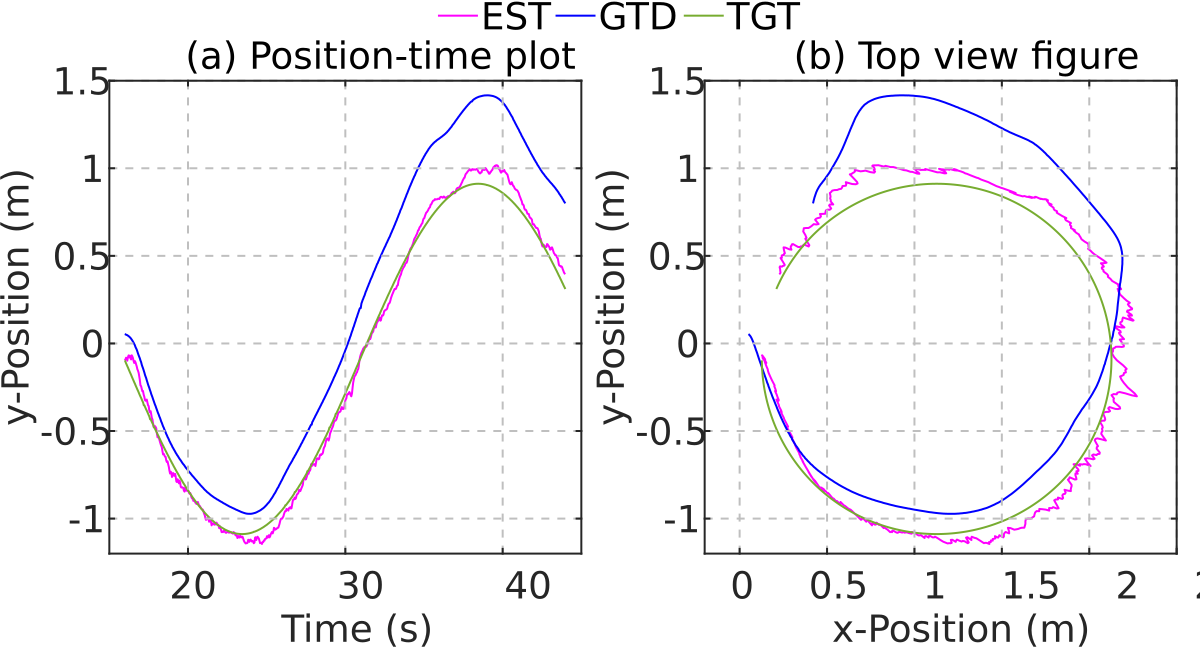}
    \vspace{-4pt}
    \caption{Illustration of estimated trajectory (EST), ground truth trajectory (GTD) and target trajectory (TGT) with speed 0.2 m/s.}
    \label{fig:control}
\end{figure}
The results are shown in Table~\ref{tab:control}. The estimation and control RMSEs here are calculated by:
\begin{subequations}
\begin{align}
\text{Estimation RMSE} &= \sqrt{\frac{1}{n}\sum_{i=1}^{n}||\boldsymbol{p}_{i,e}-\boldsymbol{p}_{i,g}||^2_2} \\
\text{Control RMSE} &= \sqrt{\frac{1}{n}\sum_{i=1}^{n}||\boldsymbol{p}_{i,e}-\boldsymbol{p}_{i,t}||^2_2}
\end{align}
\end{subequations}
where $\boldsymbol{p}_{i,e}$, $\boldsymbol{p}_{i,g}$ and $\boldsymbol{p}_{i,t}$ are estimated, ground truth and target position, respectively.
The results indicate that, despite the controller exhibiting errors of a decimeter, the estimator can still deliver satisfactory online estimation accuracy.
As shown in Fig.~\ref{fig:control}, the estimated trajectory captures the overall trend of the target trajectory, while the actual trajectory remains affected by drift, especially along the y-axis.
This drift may be stemmed from the controller’s limited ability to compensate for environmental spatial heterogeneity.
On the other hand, the characteristics of single anchor UWB is that it usually only calibrates one dimension, as discussed before, which may further contribute to the observed precision degradation in the y-direction.
Overall, these experiments sufficiently demonstrate the capability of RASWE to support real-world closed-loop control applications.

\subsection{Simulation of Covariance and Aerial Drag Estimation}
Since the ground truth of noise covariance and aerial drag matrices is inaccessible in real-world experiments, we perform simulations in MATLAB\textsuperscript{\textregistered} to provide a direct numerical evaluation of their estimation.
All parameters in the left column of Table~\ref{tab:param} remain unchanged. For the initial values in the right column, $\check{\boldsymbol{P}}_0$, $\hat{\phi}_0$, and $\hat{\psi}_0$ are preserved, while all others are initialized with their ground truth values. The UWB anchor is placed at $(0, 0, 0)$, and the UAV starts from $(1, 0, 0.2)$ at rest.
The simulation runs for $2000$ timesteps with an interval $\text{dt} = 0.04$, and the first additional $20$ timesteps are used to initialize the estimator.
The actual acceleration $\boldsymbol{i}_k$ in \eqref{31motionModel} is defined as follows:
\begin{equation}
    \label{45acc}
    \boldsymbol{i}_k = \begin{bmatrix}
        -\pi\sin(k\cdot\text{dt}/12)/2.4\\
        \pi\cos(k\cdot\text{dt}/12)/2.4\\
        0.05\cos(k\cdot\text{dt}/24)
    \end{bmatrix}-\boldsymbol{v}_{k-1}
\end{equation}
where $k=1,2,\cdots$ is the current timestep.
The actual aerial drag matrix $\boldsymbol{\mu}_k$ is set as $\text{diag}(1+0.03\sin\frac{k\pi}{200},1+0.03\sin\frac{k\pi}{250},1+0.03\sin\frac{k\pi}{225})$.
The actual covariance matrices are defined as follows:
\begin{subequations}
\label{45cov}
\begin{align}
    \boldsymbol{Q}_k&=\frac{10+9\sin\frac{k\pi}{275}}{2500}(\text{diag}(7,3,1,4,9,1)+\boldsymbol{\Lambda}_6)\label{45cov1}\\
    \boldsymbol{R}_k&=\frac{1.5+1.2\sin\frac{k\pi}{325}}{2000}(\text{diag}(9,5,4,1)+\boldsymbol{\Lambda}_4)\label{45cov2}
\end{align}
\end{subequations}
where $\boldsymbol{\Lambda}_n$ represents an $n \times n$ matrix with entries $\Lambda(i,j) = 0.1$ if $i+j$ is even, and $0.2$ otherwise.
The simulation is repeated 100 times with different random seeds.\par
As discussed previously, the augmented Kalman filter \eqref{32KFforward} combined with the backward smoother \eqref{33backward} is mathematically equivalent to a MAP estimator \cite{roboEst}, where the covariance matrices serve as relative weights in a least-squares problem.
Since scaling both sides of the equation by the same non-zero factor does not affect the solution, our focus is only on the relative magnitudes, i.e., the weight distribution.
We utilize the "softmax" operation to form this weight distribution, normalizing each entry into the probability, for example, $\text{Prob}[Q(ij)]=\exp Q(ij)/ \sum_{m=1}^{6}\sum_{n=1}^{6}\exp Q(mn)$.
Since the diagonal of covariance is more significant for the estimator, we additionally evaluate the diagonal weight distribution.\par
In this way, the challenge of assessing the covariance similarity is reformulated as evaluating the "distance" between two weight distributions.
Hence, we employ Kullback–Leibler divergence (KLD) \cite{kullback1951information}:
\begin{equation}
    \label{45kld}
    D_{\text{KL}}(p||q)=\sum_{x\in\chi}P(x)\ln(p(x)/q(x))
\end{equation}
where $p$ and $q$ are two distributions with shared sample space $\chi$.
For instance, $D_{\text{KL}}(\mathrm{N}(0,1)||\mathrm{N}(0,2)) \simeq 0.3181$, while $D_{\text{KL}}(\mathrm{N}(0,0.99)||\mathrm{N}(0,1.01))\simeq3.947\times10^{-4}$.
The mean KLD values of $\boldsymbol{Q}$ and $\boldsymbol{R}$, for both diagonals and full matrices, averaged over all 2000 timesteps and 100 simulations, are reported in Table~\ref{tab:evalQRD}.
To ensure consistency, the simulated tasks are designed with a difficulty level comparable to the real-world experiments, yielding RMSEs similar to $\text{B}_1$--$\text{B}_3$ in Table~\ref{tab:overall}.
The results show that the estimated weight distributions of both $\boldsymbol{Q}$ and $\boldsymbol{R}$ reasonably approximate the true distributions, with $\boldsymbol{R}$ achieving closer alignment.
This confirms our method’s effectiveness in capturing the statistical characteristics of noise and in finding appropriate weights to balance prior knowledge with sensor measurements.\par
For aerial drag estimation, we evaluate the relative error between the estimated value $\hat{\mu}_{t}(ii)$ and true value $\bar{\mu}_{t}(ii)$ of the main diagonal, since the actual off-diagonal elements are zero and thus not suitable for ratio-based comparison.
The relative RMSE for a single simulation is computed as: $(\sum_{t=0}^{2000}\sum_{i=1}^3\frac{\hat{\mu}_{t}(ii)-\bar{\mu}_{t}(ii)}{\bar{\mu}_{t}(ii)}\times100\%)^{0.5}$, 
which is then averaged over $100$ simulations.
As shown in Table~\ref{tab:evalQRD}, the relative RMSE remains at the percentage level, further demonstrating the proposed method’s ability to effectively model aerial drag dynamics.
These simulation results suggest that our method reasonably captures the statistical properties of time-varying noise and models aerial drag effects with acceptable accuracy, consistent with the findings from our real-world experiments.
\begin{table}[htbp]
    \centering
    \vspace{-10pt}
    \setlength{\tabcolsep}{7.2mm}
    \caption{Evaluation of Covariance and Aerial Drag Estimation in Simulation}
    \begin{tabularx}{0.98\linewidth}{lc}
        \toprule
        \multicolumn{1}{c}{\textbf{Item}}& \textbf{Average Value}\\
        \midrule
        RMSE of position estimation & $0.13824$ \\
        $D_{\text{KL}}$ of $\boldsymbol{Q}$'s diagonal entries & $3.245\times10^{-3}$ \\
        $D_{\text{KL}}$ of $\boldsymbol{Q}$'s whole matrix & $5.899\times10^{-3}$ \\
        $D_{\text{KL}}$ of $\boldsymbol{R}$'s diagonal entries & $2.537\times10^{-4}$ \\
        $D_{\text{KL}}$ of $\boldsymbol{R}$'s whole matrix & $3.136\times10^{-4}$ \\
        Aerial drag relative RMSE & $6.492$\% \\
        \bottomrule
    \end{tabularx}
    \label{tab:evalQRD}
\end{table}

\section{Conclusion}\label{secConc}
In this paper, we propose a restricted adaptive sliding window estimator using a single anchor for UAV positioning in harsh environments, which simultaneously estimates the states, noise covariance matrices and aerial drag.
Our work differs from existing adaptive estimators in two ways.
First, we introduce an error propagation matrix to assess estimation performance online, and we accordingly impose inherent restrictions to mitigate potential divergence caused by sensor malfunctions.
Secondly, we develop an adaptive aerial drag estimator to adjust the motion model dynamically, enhancing overall performance. 
Experiments validate the effectiveness of the estimator. 
Our proposed method achieves an average RMSE of 0.17m in common environments and 0.39m in harsh environments, outperforming the state-of-the-art.
The covariance estimation strongly tracks the statistic features of dynamics and observation model noises, which overcomes the asymmetry of single anchor estimation.
The error propagation restriction demonstrates its proficiency in suppressing divergence when handling abrupt changes like sharp turns and sudden trajectory fluctuations.
The real-world close-loop control experiments further underscore the estimator's competence in practical tasks.\par
In the future, we will enhance the adaptability of the estimator, considering lumped vibrations from the flight controller.

\bibliography{IEEEabrv,bibliography}
\bibliographystyle{IEEEtran}

\newpage

\begin{IEEEbiography}
[{\includegraphics[width=1in,height=1.25in,clip,keepaspectratio]{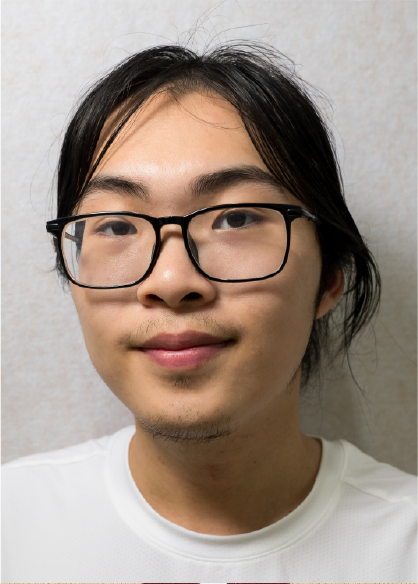}}] 
{Kaiwen Xiong} is currently pursuing the undergraduate degree in automation at Shanghai Jiao Tong University, Shanghai, China. \par
He is currently with the State Key Laboratory of Mechanical System and Vibration, School of Mechanical Engineering. His research interests include state estimation and adaptive positioning of unmanned systems.
\end{IEEEbiography}

\vspace{-10cm} 

\begin{IEEEbiography}
[{\includegraphics[width=1in,height=1.25in,clip,keepaspectratio]{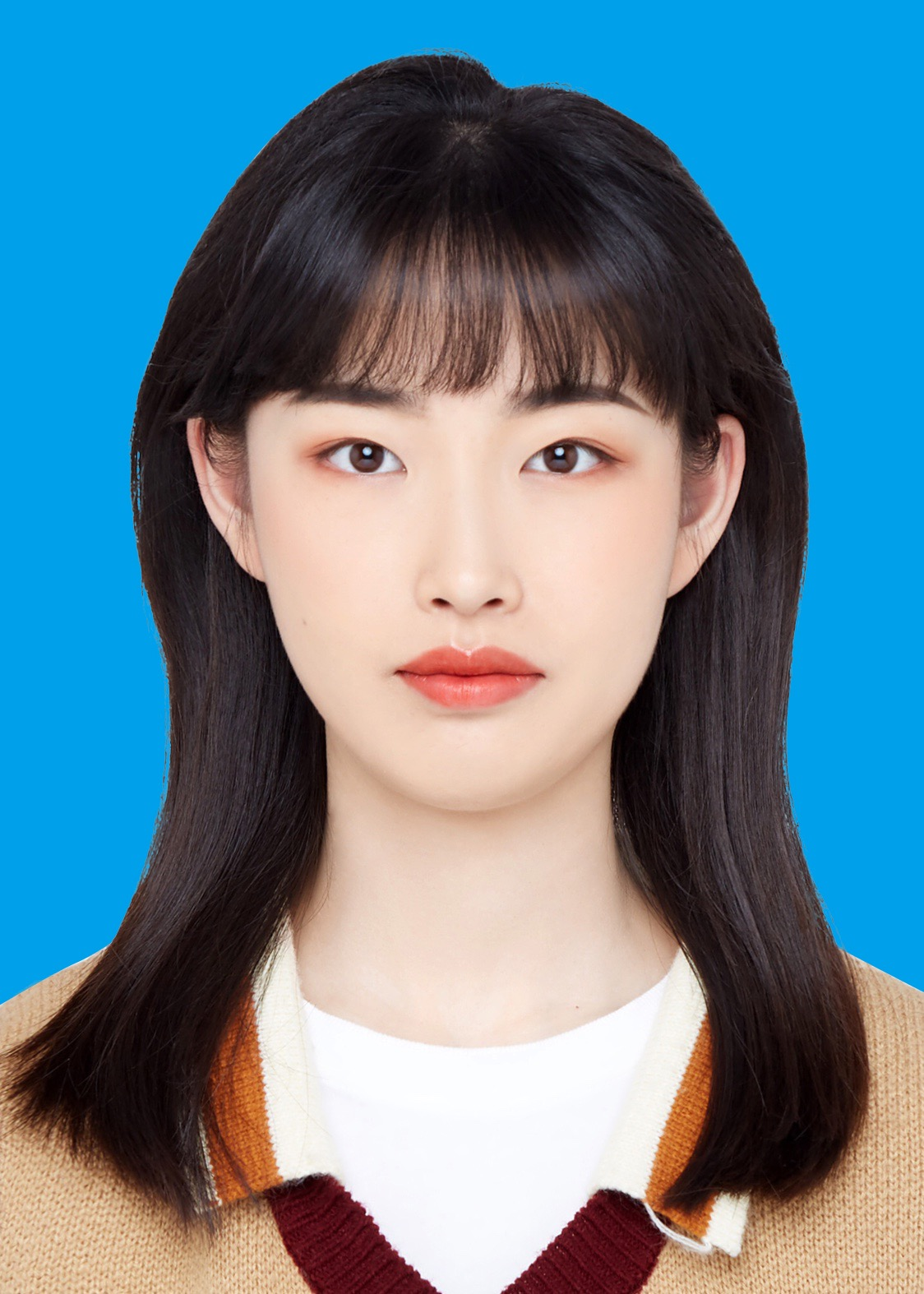}}] 
{Sijia Chen} received the B.S. degree in mechanical design manufacture and automation from the University of Electronic Science and Technology of China, Sichuan, China, in 2022. She is currently a Ph.D. candidate with the State Key Laboratory of Mechanical System and Vibration, School of Mechanical Engineering, Shanghai Jiao Tong University.\par
Her research interests include state estimation and intelligent control of unmanned systems.
\end{IEEEbiography}

\vspace{-10cm} 

\begin{IEEEbiography}[{\includegraphics[width=1in,height=1.25in,clip,keepaspectratio]{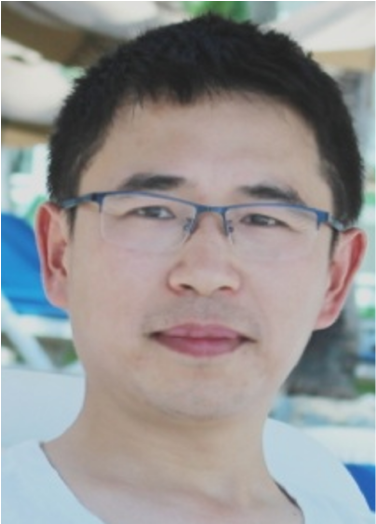}}]  
{Wei Dong} received the B.S. degree and Ph.D. degree in mechanical engineering from Shanghai Jiao Tong University, Shanghai, China, in 2009 and 2015, respectively.\par 
He is currently an associate professor in the Robotic Institute, School of Mechanical Engineering, Shanghai Jiao Tong University. For years, his research group was champions in several national-wide autonomous navigation competitions of unmanned aerial vehicles in China. In 2022, he was selected into the Shanghai Rising-Star Program for distinguished young scientists. His research interests include cooperation, perception and agile control of unmanned systems.
\end{IEEEbiography}

\end{document}